\documentclass{article}

\usepackage[preprint]{neurips_2026}

\usepackage[utf8]{inputenc} 
\usepackage[T1]{fontenc}    
\usepackage{hyperref}       
\usepackage{url}            
\usepackage{booktabs}       
\usepackage{amsfonts}       
\usepackage{nicefrac}       
\usepackage{microtype}      
\usepackage{xcolor}         

\usepackage{natbib} 
\usepackage{mathtools} 
\usepackage{booktabs} 
\usepackage{tikz} 

\usepackage{adjustbox}
\usepackage{microtype}
\usepackage{graphicx}
\usepackage{subcaption}
\usepackage{booktabs}
\usepackage{hyperref}
\usepackage{amsmath}
\usepackage{amssymb}
\usepackage{mathtools}
\usepackage{amsthm}
\usepackage{xspace}
\usepackage{enumitem}
\usepackage{algorithm2e}
\usepackage{multirow}

\renewcommand{\d}{{\rm d}}  
\newcommand{\e}{{\bf e}}

\newcommand{\Dcal}{\mathcal{D}}

\newcommand{\Hcal}{{\mathcal{H}}}

\newcommand{\ben}{\begin{enumerate}}
\newcommand{\een}{\end{enumerate}}

\newcommand{\EE}{\mathbb{E}}

\newcommand{\cmt}[1]{}

\theoremstyle{plain}
\newtheorem{theorem}{Theorem}[section]

\newtheorem{lemma}[theorem]{Lemma}

\theoremstyle{definition}

\theoremstyle{remark}

\newcommand{\ours}{SNMPP\xspace}

\title{Structured Neural Marked Point Processes for Interpretable Event Interaction Modeling}

\author{ 
Zhitong Xu$^{1}$, Qiwei Yuan$^{1}$, Yinghao Chen$^{1}$, Shandian Zhe$^{1,2}$, and Bin Shen$^{2}$ \\
$^1$ Kahlert School of Computing, University of Utah \\ 
$^2$ Celonis AI \\}

\begin{document}

\maketitle
\begin{abstract}
Multi-class event streams arise in numerous real-world applications, where uncovering structured, interpretable inter-event relationships, together with accurate prediction, remains a central challenge. Existing neural point process models are highly expressive but encode event interactions in a black-box manner, preventing explicit discovery of structured dependencies.
In this paper, we propose a structured neural marked point process (\ours) that achieves high modeling flexibility while enabling explicit event-wise and class-wise relationship discovery from data. Our model constructs a product-form neural influence kernel composed of a signed interaction network over event types and a delay-aware monotonic temporal network. This design enables explicit characterization of inter-class influence topology --- including excitation, inhibition, and neutrality --- while flexibly capturing diverse temporal decay patterns and potential influence delays. For efficient learning, we develop a stratified Monte Carlo estimator for stochastic training. Extensive experiments on synthetic and real-world benchmark datasets validate the ability of our approach to uncover structured relationships and deliver strong predictive performance.


\end{abstract}
\section{Introduction}
Multi-class event streams arise in many real-world scenarios, including social networks, online retail platforms, supply-chain systems, medical visit records, and enterprise processes. Modeling such event sequences to uncover structured inter-event dependencies and enable accurate prediction is central to understanding the dynamics of event occurrences. Reliable prediction further supports risk monitoring, early warning, and timely intervention.

Many temporal point process models have been proposed for event modeling. Classical Poisson processes assume event independence and therefore ignore interactions between events. Hawkes processes~\citep{hawkes1971spectra} extend this framework by allowing past events to excite future occurrences through parametric triggering kernels, such as the exponential kernel. However, standard Hawkes models are restricted to excitatory effects and do not capture inhibition. Moreover, most formulations assume that an event’s influence is strongest immediately after its occurrence and then decays over time. In many real-world applications, however, event influences may exhibit delayed effects, initially weak and peaking after a certain time lag, as observed in delayed medical treatment effects~\citep{holford2018pharmacodynamic,meng2021impact} or supply-chain replenishment processes~\citep{heydari2009study,chang2019effect}.

Recent neural point processes have emerged as a dominant framework for event modeling. These models directly parameterize the conditional intensity using deep neural network architectures. For example, the Neural Hawkes Process (NHP)~\citep{mei2017neural} and Recurrent Marked Temporal Point Process (RMTPP)~\citep{du2016recurrent} encode event histories via recurrent neural networks and model the intensity as a transformation of the  hidden states. The Transformer Hawkes Process (THP)~\citep{zuo2020transformer} and Self-Attentive Hawkes Process (SAHP)~\citep{zhang2020self} treat events as tokens and employ causal self-attention mechanisms to obtain contextualized embeddings for intensity computation. Although these approaches substantially enhance representational capacity, event relationships are encoded implicitly within latent representations, hindering explicit and interpretable characterization of interaction structure.

To address these limitations, we propose \ours, a structured neural marked point process. Our framework achieves high modeling flexibility while enabling explicit discovery of event-wise and class-wise interaction relationships from data. Our major contributions are as follows:
\begin{itemize}
\item \textbf{Model.} We construct a product-form neural influence kernel composed of a signed interaction network and a delay-aware monotonic temporal network. The signed interaction network takes pairs of event-type embeddings as input and outputs a signed interaction strength, reflecting the influence topology between event types --- including excitation, inhibition, and neutrality. The temporal network is designed to be monotonically decreasing with outputs constrained to the bounded range $[0,1]$, focusing on modeling the temporal decay of influence. To ensure monotonicity, we incorporate input negation, nonnegative network weights, monotonic activation functions, and soft output clipping. The temporal input is defined as the absolute time difference offset by a learnable delay parameter, enabling the model to capture both delayed peak effects and flexible decay patterns over time.
\item \textbf{Algorithm.} For efficient training, we develop a stratified Monte Carlo sampling scheme to provide an unbiased estimator of the likelihood integral terms. Specifically, each inter-event interval is partitioned into multiple segments, from which one sample is drawn per segment to estimate the integral. This stratification reduces variance while preserving sensitivity to local variations in the intensity function. 
\item \textbf{Experiments.}
We first evaluated \ours on two synthetic temporal point process datasets featuring delayed triggering and inhibitory interactions. Our method accurately recovers the conditional intensities, inter-class influence topology, delay parameters, and influence kernel shapes. We next assessed the predictive performance of \ours on five real-world benchmark datasets. Our model consistently improves next-event time prediction while achieving competitive event-type prediction accuracy compared to state-of-the-art neural point process models. Finally, we examined \ours on a simulated supply-chain system governed by physical constraints, decision rules and stochastic dynamics, rather than a predefined point process model. Despite the non–point-process data-generating mechanism, our method successfully uncovers the induced interaction structure, delayed effects, and gating behaviors underlying the simulation, while accurately recovering critical lead-time parameters.
\end{itemize}

\section{Preliminaries: Marked Temporal Point Processes}

We consider a temporal point process with $K$ event types (marks). Let $\mathcal{H}_t = \{(t_n, k_n) : t_n < t\}$
denote the event history up to time $t$, where $t_n$ is the time of the $n$-th event and $k_n \in \{1,\ldots,K\}$ is its mark. The marked conditional intensity function for  type $k$ is defined as
\begin{align}
&\lambda_k(t \mid \mathcal{H}_t) =
\lim_{\Delta t \to 0}
\frac{
\mathbb{P}\big(
\text{an event of type } k \text{ occurs in } [t,t+\Delta t)
\mid \mathcal{H}_t
\big)
}{
\Delta t
}. \notag 
\end{align}

The total conditional intensity is $\lambda(t \mid \mathcal{H}_t)= \sum_{k=1}^K
\lambda_k(t \mid \mathcal{H}_t)$. 
The conditional probability that the next event at time $t$ has mark $k$ is
$\mathbb{P}(k \mid t, \mathcal{H}_t)
=
\frac{
\lambda_k(t \mid \mathcal{H}_t)
}{
\lambda(t \mid \mathcal{H}_t)
}$.

Let $\Gamma = \{(t_n, k_n)\}_{n=1}^N$ be a sequence of marked events observed over $[0,T]$. The log-likelihood of the observed sequence $\Gamma$ under the model is  
\begin{align}
\log p(\Gamma)
=
\sum\nolimits_{n=1}^N
\log \lambda_{k_n}(t_n \mid \mathcal{H}_{t_n})
-
\int_0^T
\lambda(t \mid \mathcal{H}_t)
\, dt. \notag 
\end{align}



\begin{figure*}
	\centering
	\includegraphics[width=\linewidth]{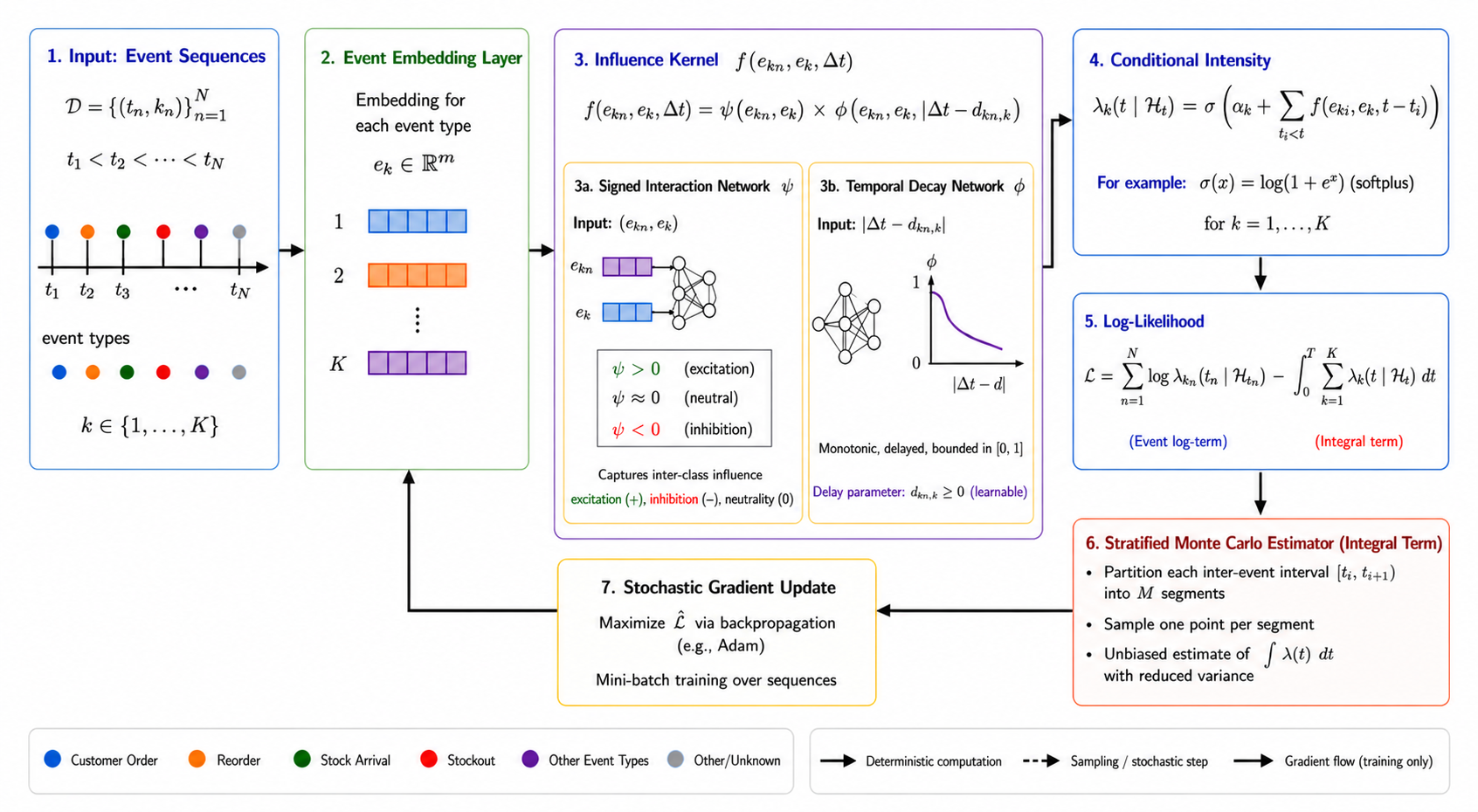}
    \caption{\small Graphical Illustration of \ours.}
	\label{fig:diagram}
\end{figure*}

\section{Methodology}

\subsection{Model}
To enable explicit event-wise and type-wise interaction discovery while retaining flexibility to capture complex temporal dependencies, we model the marked conditional intensity as
\begin{align}
    &\lambda_k(t\mid \Hcal_t)  = \sigma\left( \alpha_k + \sum\nolimits_{t_n<t} f_{k_n \rightarrow k}(t - t_n)\right),  \label{eq:lam-k}
\end{align}
where $\alpha_k$ is a latent baseline parameter capturing the spontaneous tendency of event type $k$, and $f_{k_n \rightarrow k}(\Delta t)$ denotes the influence kernel that characterizes how a past event of type $k_n$ affects the intensity of type $k$ events after a time lag $\Delta t = t - t_n$. A positive value of $f_{k_n \rightarrow k}(\Delta t)$ indicates excitation, a negative value indicates inhibition, and values close to zero correspond to negligible influence. To ensure non-negativity of the conditional intensity, we apply a positive link function $\sigma(\cdot)$ to the superposition of the baseline and interaction terms. 

Next, we introduce a learnable embedding representation $\e_k \in \mathbb{R}^m$ for each event type $k$, and model the influence kernel as a function of the corresponding embeddings and the time lag: $f_{k_n\rightarrow k}(\Delta t) = f(\e_{k_n}, \e_k, \Delta t)$. 
This framework enables a unified modeling of interactions across all event types, requiring only a single function $f$ to capture pairwise influences. In contrast to learning $K^2$ separate kernels for each ordered type pair, our embedding-based framework reduces model complexity and improves scalability, particularly when the number of event types $K$ is large. 

To explicitly capture both type-wise interaction structure and event-wise temporal influence in a flexible and data-driven manner, we further decompose the influence kernel as a product of two neural components:
\begin{align}
    &f(\e_{k_n}, \e_k, \Delta t) = \psi(\e_{k_n}, \e_k) \cdot \phi(\e_{k_n}, \e_k, |\Delta t - d_{k_n, k}|). \label{eq:inf-kernel}
\end{align}
Here, $\psi(\e_{k_n}, \e_k)$ is an interaction network that takes the embeddings of the source and target event types as input and outputs a signed magnitude representing the overall interaction strength and polarity. A positive value indicates excitation (i.e., a triggering effect), a negative value indicates inhibition, and a value near zero corresponds to negligible interaction. The magnitude reflects the maximal influence strength between the two event types. 
This formulation allows $\psi$ to induce an asymmetric interaction topology among event types, capturing directed relationships without requiring separate parameterizations for each ordered pair.

The network $\phi(\e_{k_n}, \e_k, \Delta t)$ in~\eqref{eq:inf-kernel} serves as a temporal kernel that models how interaction strength evolves over time. The embeddings of the source and target event types provide \textit{contextual information}, while $\phi$ focuses on capturing temporal variation. It is natural to assume that the influence of past events attenuates over time. To flexibly model a broad class of decay patterns beyond restrictive parametric forms (e.g., exponential decay), we construct $\phi$ as a \textit{monotonically decreasing} neural network with respect to the time lag.

Specifically, we adopt a feed-forward architecture in which all linear-layer weights are constrained to be non-negative and activation functions are chosen to be monotonically non-decreasing (e.g., Tanh or SoftPlus). By the chain rule, the resulting network is monotonic non-decreasing with respect to its input. To obtain a monotonically \textit{decreasing} function in $\Delta t$, we feed $-\Delta t$ into the first layer, ensuring $\frac{\partial \phi}{\partial \Delta t} < 0$.

To ensure that $\phi$ captures only temporal variation --- leaving the overall interaction magnitude and sign to the $\psi$ network --- we restrict its output to a bounded interval $[a,b]$. This separation avoids scale confounding between the two neural components and improves interpretability of the learned interactions. While hard clipping via $x \mapsto \min(\max(x,a),b)$ enforces this constraint, it blocks gradients when the output falls outside the interval, potentially hindering optimization. Instead, we use a differentiable soft-clipping transformation:
\begin{align}
    \text{clip}_s(x;a,b)=s\log(e^{x/s} + e^{a/s}) - s \log(e^{(x-b)/s}+1), \label{eq:soft-clilp} 
\end{align}
where $s>0$ controls smoothness. As $s \to 0$, this function converges to $\min(\max(x,a),b)$. A detailed derivation of this smooth approximation, including its connection to log-sum-exp relaxations of the max/min operators and its convergence analysis, is provided in Appendix~\ref{sect:soft-clipping}.
It is natural to set $a=0$ and $b=1$, so that the temporal response is constrained to $[0,1]$, allowing $\phi$ to capture normalized temporal variation while leaving the overall magnitude and sign to $\psi$.

To capture delayed effects that commonly arise in practical applications --- such as delayed medical treatment responses, replenishment cycles in supply chains, and policy interventions in crime reduction --- we introduce a delay parameter $d_{k_n,k} > 0$ representing the peak time of influence. Specifically, we feed $|\Delta t - d_{k_n,k}|$ into the $\phi$ network, as shown in~\eqref{eq:inf-kernel}. This formulation allows the influence to be initially weak, increase to its maximum at $d_{k_n,k}$, and  decay thereafter, yielding a unimodal temporal profile centered at $d_{k_n,k}$. The model naturally subsumes the non-delayed case as a special instance: when $d_{k_n,k}$ approaches zero, the influence peaks at the event occurrence time.


\subsection{Algorithm}\label{sect:algo}
Given a collection of observed event sequences $\Dcal$,  training maximizes the joint log likelihood, $\log p(\Dcal) = \sum\nolimits_{\Gamma \in \Dcal} \log p(\Gamma)$.
For each sequence $\Gamma = \{(t_n, k_n)\}_{n=1}^N \in \Dcal$, the log likelihood is 
\begin{align}
    &\log p(\Gamma) =\sum\nolimits_{n=1}^N \log \lambda_{k_n}(t_n|\Hcal_{t_n}) - \sum\nolimits_{k=1}^K \sum\nolimits_{n=0}^N\int_{t_n}^{t_{n+1}}  \lambda_k(t|\Hcal_{t}) \d t, \label{eq:event-ll}
\end{align}
where $t_0 = 0$ and $t_{N+1}=T$. Since no events occur within each inter-event interval $(t_n, t_{n+1})$, the history remains unchanged over that interval. By our definition $\Hcal_t=\{(t_i,k_i):t_i<t\}$, for any $t\in(t_n,t_{n+1})$ we have
$\Hcal_t = \Hcal_{t_{n+1}}$. 
Consequently, the intensity is piecewise-defined and depends only on $\Hcal_{t_{n+1}}$ within each interval.

The integral terms do not admit closed-form solutions. To address this, we propose a stratified Monte Carlo estimator for each interval. Specifically, for each interval $[t_n, t_{n+1}]$, we partition it into $Q$ equal, non-overlapping subintervals $S_1, \ldots, S_Q$ such that $\bigcup_{q=1}^Q S_q = (t_n, t_{n+1})$. We then draw one sample $\hat{t}_q$ uniformly from each segment $S_q$ and approximate the integral as
\begin{align}
    &\int_{t_n}^{t_{n+1}} \lambda_k(t|\Hcal_t) \d t = \int_{t_n}^{t_{n+1}} \lambda_k(t|\Hcal_{t_{n+1}}) \d t \approx \frac{t_{n+1} - t_n}{Q} \cdot \sum\nolimits_{q=1}^Q \lambda_k(\hat{t}_q|\Hcal_{t_{n+1}}). 
\end{align}
Compared to drawing $Q$ i.i.d. samples uniformly from the entire interval, the stratified estimator remains unbiased while typically reducing variance. 
Stratification removes the between-segment variance component present in uniform sampling and retains only within-segment variability, which is typically smaller when the intensity varies across the interval. We provide a formal variance-reduction analysis in Appendix~\ref{app:stratified-mc}, Lemma~\ref{lem:stratified-mc-var}.
By enforcing coverage across subintervals, it prevents sample concentration in small regions and better captures local variations of the intensity function.

We employ mini-batch stochastic training. At each iteration, a mini-batch of event sequences is sampled from $\Dcal$, and stratified Monte Carlo sampling is performed to approximate the integral terms for all inter-event intervals within the batch. Model parameters are optimized using stochastic gradient-based optimization. The overall method is illustrated in Figure~\ref{fig:diagram}.

\section{Related Work}
\label{sec:related}
A rich body of work has been developed for temporal point processes. Early models include Poisson processes~\citep{lawless1987regression,grandell2006doubly} and their extensions in matrix and tensor factorization settings~\citep{charlin2015dynamic,gopalan2014content,gopalan2015scalable}. 
Hawkes processes (HPs)~\citep{hawkes1971spectra} subsequently gained significant attention due to their ability to capture mutual excitation among events
~\citep{blundell2012modelling,du2015dirichlet,wang2017predicting,yang2017decoupling,xu2018benefits}.

More recently, neural point processes have emerged as a dominant framework for temporal event modeling. These models use deep neural architectures to parameterize the conditional intensity directly. For example, Recurrent Marked Temporal Point Processes (RMTPP)~\citep{du2016recurrent} employ recurrent neural networks (RNNs) to encode event histories into hidden states, from which the total intensity and mark distribution are derived. Neural Hawkes Processes (NHP)~\citep{mei2017neural} follow a related recurrent formulation but use continuous-time LSTM states~\citep{hochreiter1997long} to model marked conditional intensities. \citet{omi2019fully} proposed modeling the cumulative intensity function with monotonic networks~\citep{sill1997monotonic} conditioned on RNN states, thereby avoiding explicit numerical integration in likelihood computation.
Subsequent works, including Transformer Hawkes Processes (THP)~\citep{zuo2020transformer}, Self-Attentive Hawkes Processes (SAHP)~\citep{zhang2020self}, and Attentive Neural Hawkes Processes (AttNHP)~\citep{yang2022transformer}, treat events and their types as tokens and apply causal attention mechanisms to encode historical dependencies. The conditional intensity is then modeled as a transformation of the resulting contextual representations. More recently, \citet{yuan2025residual} proposed decomposing a point process into the superposition of a classical structured point process, such as a Hawkes process, and a neural point process. This design captures some predefined event structures while assigning the remaining unexplained dynamics to a black-box neural component.

There has also been a growing line of work on intensity-free generative models for temporal event generation, which move beyond the conventional point process likelihood framework~\citep{Shchur2020IntensityFree,lin2022exploring,ludke2023add,ludke2026editbased,kerrigan2026eventflow}. These methods are highly flexible for generation, but they do not directly provide calibrated conditional intensities or explicit event-structure discovery. Finally, \citet{xueeasytpp} introduced an open-source benchmarking framework that implements many state-of-the-art neural point process models, providing a standardized platform for comparison and evaluation. In contrast to these approaches, our method preserves the flexibility of neural point processes while introducing an explicit neural influence kernel for interpretable event-wise and class-wise interaction modeling.

\section{Experiments}
\subsection{Synthetic Data}\label{sect:expr-syn}
We first evaluated \ours on synthetic datasets to validate its ability to capture different influence types and interaction patterns. We constructed two temporal point processes:

    \paragraph{PP1} consists of two event types, denoted as $E_1$ and $E_2$. $E_1$ events occur at a constant rate and exhibit a delayed triggering effect on $E_2$ events: 
    \begin{align}
        &\lambda_{E_1}(t|\Hcal_t) = 0.5, \;\; \lambda_{E_2}(t|\Hcal_t) = 0.05 + \sum\nolimits_{t_n<t, k_n=E_1}  0.6 \cdot \exp\left(-\frac{(t-t_n - 1.0)^2}{2\cdot 0.5^2}\right).
    \end{align}
    \paragraph{PP2} is similar to \textbf{PP1}, except that $E_1$ events exert a delayed inhibitory effect on $E_2$ events:
    \begin{align}
        &\lambda_{E_1}(t|\Hcal_t) = 0.5, \;\; \lambda_{E_2}(t|\Hcal_t) = \sigma\Big[1.0 - \sum\nolimits_{t_n<t, k_n=E_1}  1.5 \cdot \exp\left(-\frac{(t - t_n - 1.0)^2}{2 \cdot 0.5^2}\right) \Big], 
    \end{align}
    where $\sigma(x) = \frac{1}{\beta} \log(1 + \exp(\beta x))$ denotes the softplus transformation used to ensure non-negative intensity, with $\beta = 10$.
We set the time horizon to $T = 50$ for \textbf{PP1} and $T = 40$ for \textbf{PP2}. For each process, we generated 6,000 independent sequences for training and 200 sequences for validation using Ogata’s thinning algorithm~\citep{ogata1981lewis}.

We compared \ours with the following popular and state-of-the-art temporal point process models:
(1) Multivariate Hawkes Process (MHP), which models excitation effects between event types using an exponential triggering kernel;
(2) Neural Hawkes Process (NHP)~\citep{mei2017neural}, which encodes historical events via an LSTM and parameterizes the conditional intensity as a transformation of the hidden state; and
(3) Transformer Hawkes Process (THP)~\citep{zuo2020transformer}, which treats events as tokens, aggregates historical information through causal attention, and models the conditional intensity based on the resulting token representations. Our method is implemented in PyTorch and trained using the AdamW optimizer~\citep{loshchilovdecoupled2019} with a learning rate of $10^{-3}$. The mini-batch size is set to 128. Detailed hyperparameter settings are provided in Appendix~\ref{sect:hyperparameter}. We use the SoftPlus transformation with $\beta=10$ in~\eqref{eq:lam-k} to ensure non-negative conditional intensity. MHP is trained using Adam with the same learning rate. For the remaining baselines, we use the official open-source implementations with their default hyperparameter settings.
\begin{figure*}[!h]
\centering
\begin{subfigure}{.5\textwidth}
  \centering
  \includegraphics[width=\textwidth]{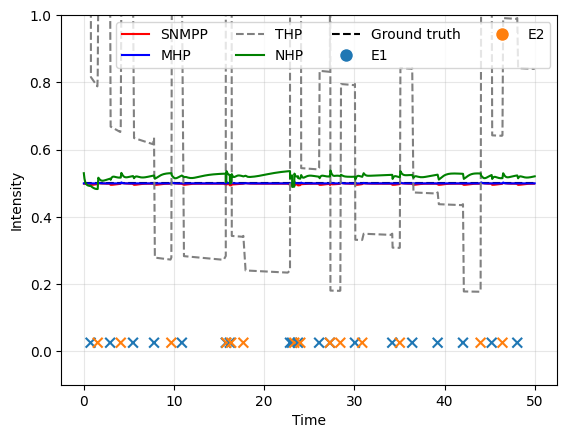}
  \caption{Event type $E_1$}
  \label{fig:intensity-pp1-e0}
\end{subfigure}%
\begin{subfigure}{.5\textwidth}
  \centering
  \includegraphics[width=\textwidth]{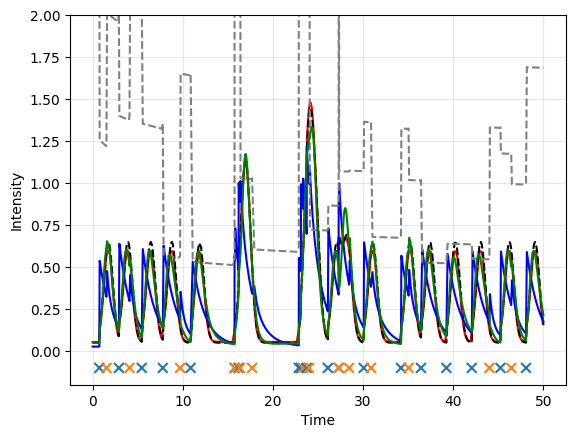}
  \caption{Event type $E_2$}
  \label{fig:intensity-pp1-e1}
\end{subfigure}
\caption{The conditional intensity function on a validation sequence from \textbf{PP1}.}
\label{fig:intensity-pp1}
\end{figure*}

\begin{figure*}[!h]
\centering
\begin{subfigure}{.5\textwidth}
  \centering
  \includegraphics[width=\textwidth]{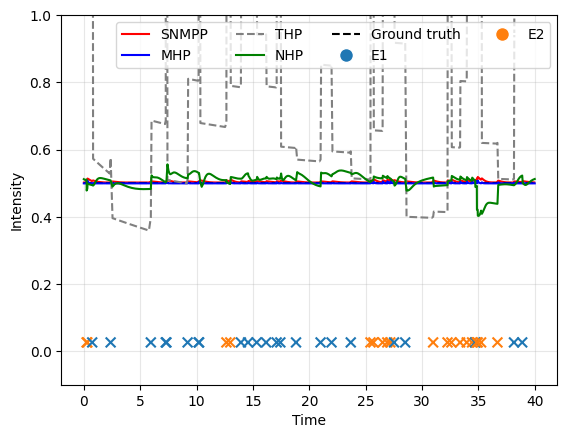}
  \caption{Event type $E_1$}
  \label{fig:intensity-pp2-e0}
\end{subfigure}%
\begin{subfigure}{.5\textwidth}
  \centering
  \includegraphics[width=\textwidth]{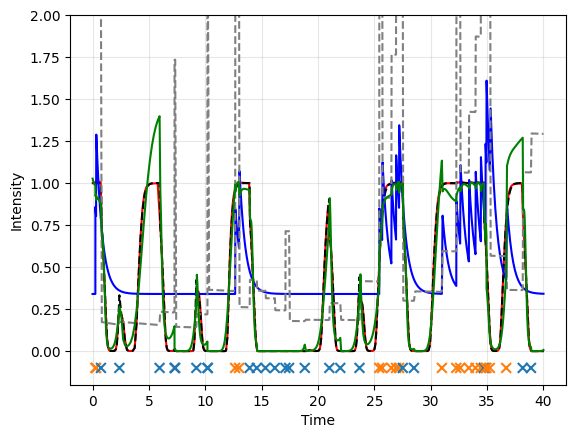}
  \caption{Event type $E_2$}
  \label{fig:intensity-pp2-e1}
\end{subfigure}
\caption{The conditional intensity function on a validation sequence from \textbf{PP2}.}
\label{fig:intensity-pp2}
\end{figure*}
\textbf{Intensity Recovery.}
We first  evaluated whether \ours can accurately recover the conditional intensity functions. For each of \textbf{PP1} and \textbf{PP2}, we randomly select one validation sequence and visualize the learned conditional intensity for each event type.  As shown in Figures~\ref{fig:intensity-pp1} and~\ref{fig:intensity-pp2}, \ours closely matches the ground-truth intensities, while THP exhibit substantial deviations. NHP approximates the intensity reasonably well, demonstrating its strong representational capacity. However, it is less accurate than \ours in capturing the complex temporal patterns of the $E_2$ intensity, particularly under delayed excitation and inhibition settings. 
MHP accurately recovers the $E_1$ intensity --- which is constant and corresponds to a homogeneous Poisson process --- in both \textbf{PP1} and \textbf{PP2}. This behavior is expected, as the Hawkes formulation reduces to a homogeneous Poisson process when excitation terms vanish. However, its estimation of the $E_2$ intensity deteriorates substantially ---  performing worse than NHP --- especially for \textbf{PP2}, where inhibition dominates the process.  This highlights the limitations of standard Hawkes processes in modeling delayed or inhibitory dynamics. Importantly, although neural baselines such as NHP and THP can approximate the intensity function, they do \textit{not} provide explicit or interpretable interaction structures, as event relationships remain embedded within latent representations.

\begin{figure*}[!h]
	\centering
	\includegraphics[width=\textwidth]{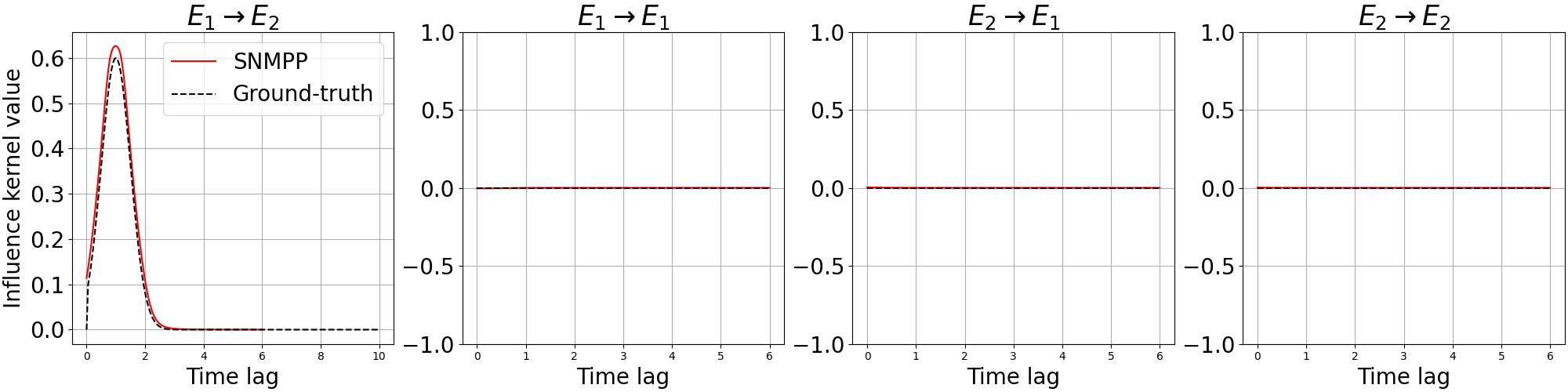}
	\caption{\small Learned influence kernel on \textbf{PP1}. Note that \ours learns a single unified influence kernel shared across all event types, as defined in~\eqref{eq:inf-kernel}.}
	\label{fig:influence-kernel-pp1}
\end{figure*}

\begin{figure*}[!h]
	\centering
	\includegraphics[width=\textwidth]{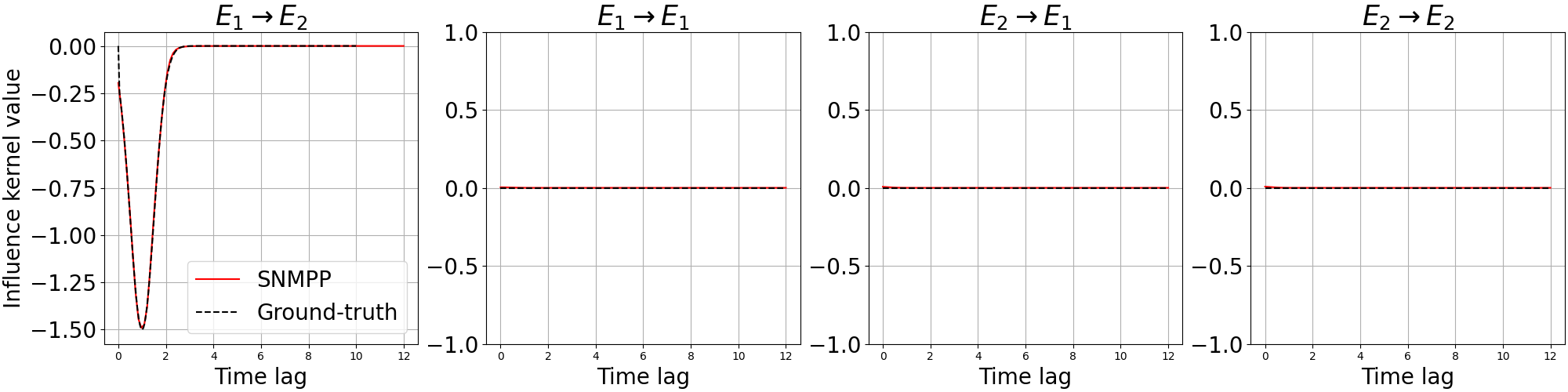}
	\caption{\small Learned influence kernel from \textbf{PP2}.}
	\label{fig:influence-kernel-pp2}
\end{figure*}
\textbf{Structure Discovery.}
We then examined whether \ours can recover the interaction structure among event types and the temporal evolution of influence strengths. To this end, we visualize the learned influence functions for every ordered pair of event types. As shown in Figures~\ref{fig:influence-kernel-pp1} and~\ref{fig:influence-kernel-pp2}, \ours not only correctly identifies the type of interaction --- excitation, inhibition, or neutrality --- but also accurately captures the temporal evolution of the influence strength.

In both \textbf{PP1} and \textbf{PP2}, the only true interaction is from $E_1$ to $E_2$: a delayed excitation in \textbf{PP1} and a delayed inhibition in \textbf{PP2}. The learned influence function $f$ closely matches the ground-truth kernel, including accurate estimation of the delay parameters. The estimated delays are $0.992$ \textit{versus} $1.0$ for \textbf{PP1}, and $0.997$ \textit{versus} $1.0$ for \textbf{PP2}.

For all other type pairs, including $E_1 \rightarrow E_1$, $E_2 \rightarrow E_2$, and $E_2 \rightarrow E_1$, there is no true interaction; the corresponding ground-truth kernels are identically zero. The learned kernels faithfully recover this null structure, with estimated delay parameters close to zero. In addition, \ours accurately estimates the baseline intensities for both processes. Detailed quantitative comparisons are provided in Appendix Table~\ref{tab:pp1-pp2}.


Overall, the results demonstrate that our \textit{unified} neural influence kernel --- despite being \textit{shared} across all event-type pairs --- can flexibly capture heterogeneous interaction patterns, distinguish delayed from non-delayed effects, and accurately recover delay parameters when present.

\begin{table*}[!htp]\centering
\caption{\small Performance of next-event prediction on real-world datasets. Numbers in parenthesis indicates the standard deviation. \textit{Results of \ours and the best result} for each metric are shown in bold.}\label{tab:pred-error}
\small
\centering
\begin{tabular}{lccccc}
\toprule
\multirow{2}{*}{MODEL} & \multicolumn{5}{c}{ Metrics (Time RMSE/Type Error Rate)} \\\cmidrule(lr){2-6}
                                & Amazon & Retweet & Taxi & MIMIC  & StackOverflow \\
\midrule
\multirow{2}{*}{MHP} & 0.635 / 75.9 \% & 22.92 / 55.7 \% & 0.382 / 9.53 \% & 0.839 / 31.7\%  & 1.388 / 65.0 \% \\
                       & (0.005 / 0.005) & (0.212 / 0.004) & (0.002 / 0.0004) & (0.023 / 0.024)  & (0.011 / 0.005) \\
\midrule
\multirow{2}{*}{RMTPP} & 0.620 / 68.1\% & 22.31 / 44.1\% & 0.371 / 9.51\% & 0.950 / 50.3\% & 1.376 / 57.3\% \\
                       & (0.005 / 0.006) & (0.209 / 0.003) & (0.003 / 0.0003) & (0.022 / 0.029) & (0.018 / 0.005) \\
\midrule
\multirow{2}{*}{NHP} & 0.621 / 67.1\% & 21.90 / 40.0\% & 0.369 / \textbf{8.50\%} & 1.021 / 44.6\% & 1.372 / \textbf{55.0\%}\\
                       & (0.005 / 0.006) & (0.184 / 0.002) & (0.003 / 0.0005) & (0.028 / 0.031) & (0.011 / 0.006)\\
\midrule
\multirow{2}{*}{SAHP} & 0.619 / 67.7\%& 22.40 / 41.6\% & 0.372 / 9.75\%  & 1.033 / 44.7\% & 1.375 / 56.1\% \\
                      & (0.005 / 0.006) & (0.301 / 0.002) & (0.003 / 0.0008) & (0.023 / 0.026) & (0.013 / 0.005) \\
\midrule
\multirow{2}{*}{THP} & 0.621 / 66.1\% & 22.01 / 41.5\% & 0.370 / 8.68\%  & 0.967 / 56.3\%& 1.374 / \textbf{55.0\%}\\
                      & (0.003 / 0.007) & (0.188 / 0.003) & (0.003 / 0.0006) & (0.020 / 0.028) & (0.021 / 0.006)\\
\midrule
\multirow{2}{*}{AttNHP} & 0.621 / \textbf{65.3\%} & 22.19 / 40.1\%& 0.371 / 8.71\%  & 0.975 / 33.0\% & 1.372 / 55.2\%\\
                                      & (0.005 / 0.006) & (0.180 / 0.003) & (0.003 / 0.0004) & (0.029 / 0.025) & (0.019 / 0.003) \\
\midrule
\multirow{2}{*}{ODETPP} & 0.620 / 65.8\% & 22.48 / 43.2\%& 0.371 / 10.54\%  & 0.996 / 38.9\% & 1.374 / 56.8\%\\
                                      & (0.006 / 0.008) & (0.175 / 0.004) & (0.003 / 0.0008) & (0.063 / 0.064) & (0.022 / 0.004) \\
\midrule
\multirow{2}{*}{IFTPP} & 0.618 / 67.5\% & 22.18 / \textbf{39.7\%}& 0.377 / {8.56\%}  & 1.046 / 21.6\% & 1.373 / {55.1\%}\\
                                      & (0.005 / 0.007) & (0.204 / 0.003) & (0.003 / 0.006) & (0.078 / 0.029) & (0.010 / 0.005) \\
\midrule
\multirow{2}{*}{SNMPP} & \textbf{0.382 / 66.6\%} & \textbf{18.60 / 40.5\%} & \textbf{0.298 / 9.21\%}  & \textbf{0.815 / 14.9 \%} & \textbf{1.05 / 56.2\%}\\
                       & (0.001 / 0.002) & (0.122 / 0.001) & (0.010 / 0.006) & (0.002 / 0.003) & (0.004 / 0.005)\\
\bottomrule
\end{tabular}
\end{table*}

\subsection{Predictive Performance}
Next, we evaluated the predictive performance of \ours on five real-world benchmark datasets:
\textbf{MIMIC}~\citep{johnson2016mimic},
\textbf{Amazon}~\citep{mcauley2018amazon},
\textbf{Retweet}~\citep{zhou2013learning},
\textbf{Taxi}~\citep{whong2014foiling}, 
and
\textbf{StackOverflow}~\citep{jure2014snap}.
These datasets span diverse application domains, including medical visits, social network interactions, online shopping behavior, taxi pick-up and drop-off events, and online question–answering platforms.  Detailed dataset descriptions and summary statistics are provided in Appendix~\ref{sect:dataset}. For all datasets except {MIMIC}, we adopt the same training, validation, and test splits as provided in EasyTPP~\citep{xueeasytpp}, an open-source benchmark suite offering standardized datasets and model implementations for temporal point processes. Since {MIMIC} is not included in EasyTPP, we use the data splits from widely adopted prior repositories~\citep{mei2017neural,zuo2020transformer}. Following~\citep{xueeasytpp}, each experiment is repeated five times with different random initializations. 
We evaluated next-event time prediction using Root Mean Square Error (RMSE) and next-event type prediction using classification error rate, reporting the mean and standard deviation across runs.

In addition to the methods introduced in Section~\ref{sect:expr-syn}, we include:
(4) Recurrent Marked Temporal Point Process (RMTPP)~\citep{du2016recurrent}, an RNN-based model that explicitly parameterizes the mark distribution;
(5) Self-Attentive Hawkes Process (SAHP)~\citep{zhang2020self}; and
(6) Attentive Neural Hawkes Process (AttNHP)~\citep{yang2022transformer}, two additional attention-based models; (7) ODETPP~\citep{chen2021neuralstpp}, which combines neural ODE dynamics for continuous intensity evolution between events with GRU updates for event-triggered intensity jumps; and 
(8) IFTPP~\citep{Shchur2020IntensityFree}, an intensity-free model that parameterizes inter-arrival time distributions instead of conditional intensities. EasyTPP provides state-of-the-art implementations of the competing models and performs systematic hyperparameter tuning based on validation performance. To ensure a fair comparison, we directly report their optimized results for these datasets. For {MIMIC}, we follow the same hyperparameter validation protocol described in~\citep{xueeasytpp}. For next-event time prediction, we use the posterior expectation under the learned conditional intensity. The expectation is evaluated via truncated numerical integration, with the truncation horizon set to five or ten times the empirical mean inter-event time of the training data. Detailed hyperparameter settings for \ours are provided in Appendix~\ref{sect:hyperparameter}.


As shown in Table~\ref{tab:pred-error}, \ours consistently achieves the lowest time-prediction error across all datasets. While its event-type prediction accuracy is not always the highest, it remains consistently competitive and close to the best-performing methods. Notably, \ours consistently outperforms RMTPP in event-type prediction and achieves accuracy comparable to, or better than, NHP, SAHP, and THP.

Overall, these results show that our neural marked point process, despite explicitly modeling structured interaction topology and temporal influence dynamics, does not compromise predictive performance. On the contrary, \ours achieves state-of-the-art time-prediction accuracy among the compared methods. One possible explanation is that the proposed structured parameterization provides an effective inductive bias, regularizing the learning problem and improving temporal generalization. We provide a more detailed discussion of the expressivity--interpretability trade-off in Appendix~\ref{sect:limitation}.

\subsection{Supply-Chain System Simulation}
\begin{figure*}[!h]
	\centering
	\includegraphics[width=\textwidth]{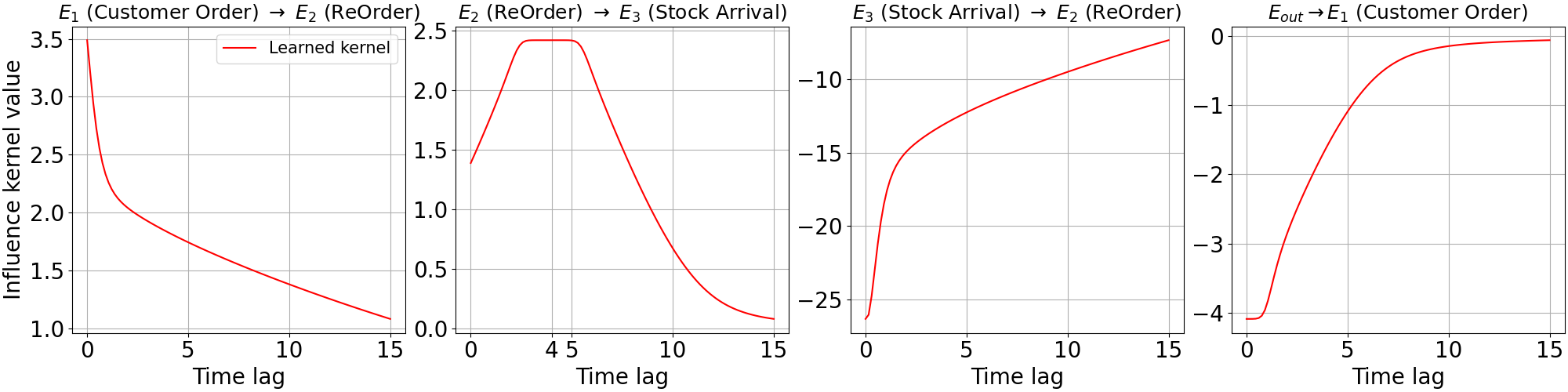}
    \caption{\small Influence kernels learned by \ours\ on event data generated by a simulated supply-chain system.}
	\label{fig:influence-kernel-sc}
\end{figure*}
Third, we evaluated \ours on a simulated supply-chain system designed to reflect realistic operational logic. The event sequences are generated from stochastic dynamics governed by physical constraints and inventory decision rules, rather than being derived from any temporal point process model. This setting enables us to assess the robustness of our approach under misspecification --- a common scenario in real-world applications --- and to examine whether \ours can still recover critical interaction patterns, temporal dynamics, and underlying operational mechanisms.

Specifically, the simulator is implemented as a hidden-state inventory process that captures a closed-loop inventory workflow with four event types: customer orders ($E_1$), replenishment orders ($E_2$), stock arrivals ($E_3$), and a recorded stockout event ($E_{\text{out}}$). Customer orders $E_1$ are generated as candidate arrivals of a homogeneous Poisson process with a rate sampled uniformly from $[1.5,3.5]$ over a time horizon $T_{\max}=30$. Each valid order decreases the (hidden) inventory by one unit. When the inventory first reaches zero, a stockout event $E_{\text{out}}$ is recorded, and thereafter customer orders are physically suppressed (i.e., not recorded) until inventory is replenished. When the inventory falls below the reorder threshold $r=5$ and no restock is pending, a replenishment order $E_2$ is triggered, which schedules a delayed stock arrival $E_3$ after a stochastic lead time sampled as $L=\max(0.5, \mathcal{N}(4.0,1))$. Upon arrival, the inventory increases by a fixed quantity $q=15$, and the stockout state is cleared. The initial inventory is set to $10$. Full simulation details are provided in Appendix Section~\ref{app:supplychain}.
Our model is trained solely on the observed event timestamps and types, without access to the underlying hidden inventory state. We generate 1,500 sequences for training and 250 sequences for validation.  

We first examined whether \ours can recover the interaction structure induced by the simulator (see Appendix~\ref{sect:structure}). We visualize the learned influence kernels for $E_1 \rightarrow E_2$, $E_2 \rightarrow E_3$, $E_3 \rightarrow E_2$, and $E_{\text{out}} \rightarrow E_1$. As shown in Figure~\ref{fig:influence-kernel-sc}, \ours correctly identifies the interaction types in all the cases. Customer orders ($E_1$) exhibit excitation on replenishment orders ($E_2$) due to inventory consumption. Replenishment orders ($E_2$) further excite stock arrival events ($E_3$). In contrast, stock arrival ($E_3$) inhibits replenishment orders ($E_2$) by restoring inventory. Finally, stock-out events ($E_{\text{out}}$) strongly inhibit customer orders ($E_1$), reflecting the physical constraint that no purchase can occur when inventory is depleted.

Notably, \ours accurately detects the \textit{delayed} excitation from $E_2$ to $E_3$. As shown in Figure~\ref{fig:influence-kernel-sc} (second plot), the temporal influence for $E_2 \rightarrow E_3$ peaks near a time lag of 4.0 --- with an estimated delay parameter of 3.92 --- exhibiting a gradual increase followed by decay. This closely matches the simulation mechanism, where lead time is sampled from $L = \max(0.5, \mathcal{N}(4.0,1))$. In practical supply chain systems, such delayed excitation is well understood: \textit{manufacturing and transportation processes introduce inherent latency, so replenishment orders do not immediately translate into stock arrivals}. For the other interaction pairs, \ours captures immediate decaying patterns without delay, consistent with the simulation design. For example, the inhibitory effect of stock-out events ($E_{\text{out}}$) on customer orders ($E_1$) is instantaneous due to physical constraints.

We further visualize the intensity curve for $E_1$ (customer orders) on a validation sequence (Appendix Figure~\ref{fig:intensity-sc}). The learned intensity reflects both the stochastic dynamics and underlying physical rules. Prior to any $E_{\text{out}}$ event, the intensity fluctuates within $[1.5, 3.5]$, consistent with the simulation that generates $E_1$ events via a Poisson process with rate uniformly sampled from this range. Upon the occurrence of an $E_{\text{out}}$ event, the intensity of $E_1$ drops sharply toward zero, reflecting the inability to place orders when inventory is exhausted. Conversely, once a stock arrival event ($E_3$) occurs, the intensity immediately returns to the baseline range, indicating that customer demand resumes upon replenishment.

Overall, these results demonstrate that \ours successfully recovers the interaction topology, delay structure, inhibitory effects, and hard gating behavior embedded in the simulator, highlighting its potential for interpretable relationship discovery in real-world systems.

\section{Conclusion}

We presented \ours, a structured neural marked point process for multi-class event streams. Our model introduces a neural influence kernel that explicitly captures inter-type influence topology and event-wise temporal influence dynamics, while allowing flexible interaction signs, temporal decay patterns, and delayed effects. Across synthetic datasets, simulated dynamical  systems, and real-world benchmarks, \ours demonstrates strong performance in both structured relationship discovery and prediction. Future work will extend the model to asymmetric temporal influence kernels and validate it on broader real-world applications; see Appendix~\ref{sect:limitation} for further discussion.



\bibliography{ref}

@article{holford2018pharmacodynamic,
  title={Pharmacodynamic principles and the time course of delayed and cumulative drug effects},
  author={Holford, Nick},
  journal={Translational and Clinical Pharmacology},
  volume={26},
  number={2},
  pages={56--59},
  year={2018},
  publisher={Korean Society for Clinical Pharmacology and Therapeutics}
}

@book{grandell2006doubly,
  title={Doubly stochastic {P}oisson processes},
  author={Grandell, Jan},
  year={2006},
  publisher={Springer}
}

@inproceedings{
Shchur2020IntensityFree,
title={Intensity-Free Learning of Temporal Point Processes},
author={Oleksandr Shchur and Marin Biloš and Stephan Günnemann},
booktitle={International Conference on Learning Representations},
year={2020},
url={https://openreview.net/forum?id=HygOjhEYDH}
}

@inproceedings{yuan2025residual,
  title={Residual {TPP}: A Unified Lightweight Approach for Event Stream Data Analysis},
  author={Yuan, Ruoxin and Fang, Guanhua},
  booktitle={Forty-second International Conference on Machine Learning},
  year={2025}
}

@article{
lin2022exploring,
title={Exploring Generative Neural Temporal Point Process},
author={Haitao Lin and Lirong Wu and Guojiang Zhao and Liu Pai and Stan Z. Li},
journal={Transactions on Machine Learning Research},
issn={2835-8856},
year={2022},
url={https://openreview.net/forum?id=NPfS5N3jbL},
note={}
}

@article{ludke2023add,
  title={Add and thin: Diffusion for temporal point processes},
  author={L{\"u}dke, David and Bilo{\v{s}}, Marin and Shchur, Oleksandr and Lienen, Marten and G{\"u}nnemann, Stephan},
  journal={Advances in Neural Information Processing Systems},
  volume={36},
  pages={56784--56801},
  year={2023}
}

@inproceedings{
ludke2026editbased,
title={Edit-Based Flow Matching for Temporal Point Processes},
author={David L{\"u}dke and Marten Lienen and Marcel Kollovieh and Stephan G{\"u}nnemann},
booktitle={The Fourteenth International Conference on Learning Representations},
year={2026},
url={https://openreview.net/forum?id=FNf9IV1P2L}
}

@inproceedings{
kerrigan2026eventflow,
title={Event{F}low: Forecasting Temporal Point Processes with Flow Matching},
author={Gavin Kerrigan and Kai Nelson and Padhraic Smyth},
booktitle={The 29th International Conference on Artificial Intelligence and Statistics},
year={2026},
url={https://openreview.net/forum?id=QXqKGOE2JW}
}

@inproceedings{charlin2015dynamic,
  title={Dynamic {P}oisson factorization},
  author={Charlin, Laurent and Ranganath, Rajesh and McInerney, James and Blei, David M},
  booktitle={Proceedings of the 9th ACM Conference on Recommender Systems},
  pages={155--162},
  year={2015}
}

@book{boyd2004convex,
  title={Convex Optimization},
  author={Boyd, Stephen and Vandenberghe, Lieven},
  year={2004},
  publisher={Cambridge University Press}
}

@article{nesterov2005smooth,
  title={Smooth minimization of non-smooth functions},
  author={Nesterov, Yurii},
  journal={Mathematical Programming},
  year={2005},
  volume={103},
  pages={127--152}
}

@inproceedings{gopalan2015scalable,
  title={Scalable recommendation with hierarchical poisson factorization.},
  author={Gopalan, Prem and Hofman, Jake M and Blei, David M},
  booktitle={UAI},
  pages={326--335},
  year={2015}
}

@article{gopalan2014content,
  title={Content-based recommendations with Poisson factorization},
  author={Gopalan, Prem and Charlin, Laurent and Blei, David M},
  journal={Advances in neural information processing systems},
  volume={27},
  year={2014}
}

@article{lawless1987regression,
  title={Regression methods for Poisson process data},
  author={Lawless, Jerald Franklin},
  journal={Journal of the American Statistical Association},
  volume={82},
  number={399},
  pages={808--815},
  year={1987},
  publisher={Taylor \& Francis}
}

@inproceedings{meng2021impact,
  title={The impact of delayed symptomatic treatment implementation in the intensive care unit},
  author={Meng, Lesley and Laudanski, Krzysztof and Restrepo, Mariana and Huffenberger, Ann and Terwiesch, Christian},
  booktitle={Healthcare},
  volume={10},
  number={1},
  pages={35},
  year={2021},
  organization={MDPI}
}

@article{chang2019effect,
  title={The effect of lead-time on supply chain resilience performance},
  author={Chang, Wei-Shiun and Lin, Yu-Ting},
  journal={Asia Pacific Management Review},
  volume={24},
  number={4},
  pages={298--309},
  year={2019},
  publisher={Elsevier}
}

@article{heydari2009study,
  title={A study of lead time variation impact on supply chain performance},
  author={Heydari, Jafar and Baradaran Kazemzadeh, Reza and Chaharsooghi, S Kamal},
  journal={The International Journal of Advanced Manufacturing Technology},
  volume={40},
  number={11},
  pages={1206--1215},
  year={2009},
  publisher={Springer}
}

@article{sill1997monotonic,
  title={Monotonic networks},
  author={Sill, Joseph},
  journal={Advances in neural information processing systems},
  volume={10},
  year={1997}
}

@inproceedings{chen2021neuralstpp,
title={Neural Spatio-Temporal Point Processes},
author={Ricky T. Q. Chen and Brandon Amos and Maximilian Nickel},
booktitle={International Conference on Learning Representations},
year={2021},
}

@inproceedings{loshchilovdecoupled2019,
  title={Decoupled Weight Decay Regularization},
  author={Loshchilov, Ilya and Hutter, Frank},
  booktitle={International Conference on Learning Representations},
  year={2019}
}

@inproceedings{xueeasytpp,
  title={EasyTPP: Towards Open Benchmarking Temporal Point Processes},
  author={Xue, Siqiao and Shi, Xiaoming and Chu, Zhixuan and Wang, Yan and Hao, Hongyan and Zhou, Fan and JIANG, Caigao and Pan, Chen and Zhang, James Y and Wen, Qingsong and others},
  year={2024},
  booktitle={The Twelfth International Conference on Learning Representations}
}

@inproceedings{du2016recurrent,
	title={Recurrent marked temporal point processes: Embedding event history to vector},
	author={Du, Nan and Dai, Hanjun and Trivedi, Rakshit and Upadhyay, Utkarsh and Gomez-Rodriguez, Manuel and Song, Le},
	booktitle={Proceedings of the 22nd ACM SIGKDD International Conference on Knowledge Discovery and Data Mining},
	pages={1555--1564},
	year={2016}
}

@inproceedings{zuo2020transformer,
	title={Transformer hawkes process},
	author={Zuo, Simiao and Jiang, Haoming and Li, Zichong and Zhao, Tuo and Zha, Hongyuan},
	booktitle={International Conference on Machine Learning},
	pages={11692--11702},
	year={2020},
	organization={PMLR}
}

@inproceedings{yang2022transformer,
  title={Transformer Embeddings of Irregularly Spaced Events and Their Participants},
  author={Yang, Chenghao and Mei, Hongyuan and Eisner, Jason},
  booktitle={Proceedings of the Tenth International Conference on Learning Representations (ICLR)},
  year={2022}
}

@inproceedings{omi2019fully,
	title={Fully neural network based model for general temporal point processes},
	author={Omi, Takahiro and Aihara, Kazuyuki and others},
	booktitle={Advances in Neural Information Processing Systems},
	pages={2122--2132},
	year={2019}
}

@inproceedings{zhang2020self,
	title={Self-attentive hawkes process},
	author={Zhang, Qiang and Lipani, Aldo and Kirnap, Omer and Yilmaz, Emine},
	booktitle={International Conference on Machine Learning},
	pages={11183--11193},
	year={2020},
	organization={PMLR}
}

@inproceedings{mei2017neural,
	title={The neural hawkes process: A neurally self-modulating multivariate point process},
	author={Mei, Hongyuan and Eisner, Jason M},
	booktitle={Advances in Neural Information Processing Systems},
	pages={6754--6764},
	year={2017}
}

@article{hochreiter1997long,
	title={Long short-term memory},
	author={Hochreiter, Sepp and Schmidhuber, J{\"u}rgen},
	journal={Neural computation},
	volume={9},
	number={8},
	pages={1735--1780},
	year={1997},
	publisher={MIT Press}
}

@inproceedings{blundell2012modelling,
  title={Modelling reciprocating relationships with Hawkes processes},
  author={Blundell, Charles and Beck, Jeff and Heller, Katherine A},
  booktitle={Advances in Neural Information Processing Systems},
  pages={2600--2608},
  year={2012}
}

@inproceedings{du2015dirichlet,
  title={Dirichlet-hawkes processes with applications to clustering continuous-time document streams},
  author={Du, Nan and Farajtabar, Mehrdad and Ahmed, Amr and Smola, Alexander J and Song, Le},
  booktitle={Proceedings of the 21th ACM SIGKDD International Conference on Knowledge Discovery and Data Mining},
  pages={219--228},
  year={2015},
  organization={ACM}
}

@article{hawkes1971spectra,
  title={Spectra of some self-exciting and mutually exciting point processes},
  author={Hawkes, Alan G},
  journal={Biometrika},
  volume={58},
  number={1},
  pages={83--90},
  year={1971},
  publisher={Oxford University Press}
}

@inproceedings{wang2017predicting,
  title={Predicting user activity level in point processes with mass transport equation},
  author={Wang, Yichen and Ye, Xiaojing and Zha, Hongyuan and Song, Le},
  booktitle={Advances in Neural Information Processing Systems},
  pages={1644--1654},
  year={2017}
}

@inproceedings{yang2017decoupling,
	title={Decoupling Homophily and Reciprocity with Latent Space Network Models.},
	author={Yang, Jiasen and Rao, Vinayak A and Neville, Jennifer},
	booktitle={UAI},
	year={2017}
}

@inproceedings{xu2018benefits,
	title={Benefits from superposed hawkes processes},
	author={Xu, Hongteng and Luo, Dixin and Chen, Xu and Carin, Lawrence},
	booktitle={International Conference on Artificial Intelligence and Statistics},
	pages={623--631},
	year={2018},
	organization={PMLR}
}

@article{ogata1981lewis,
  title={On Lewis' simulation method for point processes},
  author={Ogata, Yosihiko},
  journal={IEEE transactions on information theory},
  volume={27},
  number={1},
  pages={23--31},
  year={1981},
  publisher={IEEE}
}

@article{mcauley2018amazon,
  title={Amazon Review Data (2018)},
  author={McAuley, Julian and Ni, Jianmo},
  journal={Retrieved May},
  volume={11},
  pages={2018},
  year={2018}
}

@article{johnson2016mimic,
  title={{MIMIC-III}, a freely accessible critical care database},
  author={Johnson, Alistair EW and Pollard, Tom J and Shen, Lu and Lehman, Li-wei H and Feng, Mengling and Ghassemi, Mohammad and Moody, Benjamin and Szolovits, Peter and Anthony Celi, Leo and Mark, Roger G},
  journal={Scientific data},
  volume={3},
  number={1},
  pages={1--9},
  year={2016},
  publisher={Nature Publishing Group}
}

@inproceedings{zhou2013learning,
  title={Learning triggering kernels for multi-dimensional hawkes processes},
  author={Zhou, Ke and Zha, Hongyuan and Song, Le},
  booktitle={International conference on machine learning},
  pages={1301--1309},
  year={2013},
  organization={PMLR}
}

@article{whong2014foiling,
  title={FOILing NYC’s taxi trip data},
  author={Whong, Chris},
  journal={FOILing NYCs Taxi Trip Data. Np},
  volume={18},
  pages={14},
  year={2014}
}

@article{jure2014snap,
  title={Snap datasets: {S}tanford large network dataset collection},
  author={Jure, Leskovec},
  journal={Retrieved December 2021 from http://snap. stanford. edu/data},
  year={2014}
}

\onecolumn

\title{Appendix}
\maketitle
\appendix
\section*{Appendix}
\section{Differentiable Soft Clipping}\label{sect:soft-clipping}

To constrain the temporal response within a bounded range while preserving differentiability, we replace the hard projection
\[
\mathrm{clip}(x;a,b) = \min(\max(x,a), b)
\]
with a smooth approximation based on the log-sum-exp relaxation of the $\max$ and $\min$ operators \citep{boyd2004convex,nesterov2005smooth}.

\paragraph{Smooth Maximum and Minimum.}
The smooth approximation of $\max(x,a)$ is given by
\[
\max(x,a) \approx s \log\!\left(e^{x/s} + e^{a/s}\right),
\]
where $s>0$ controls smoothness. As $s \to 0$, this expression converges pointwise to $\max(x,a)$.

Similarly, a smooth approximation of $\min(y,b)$ can be written as
\[
\min(y,b) \approx y - s \log\!\left(e^{(y-b)/s} + 1\right).
\]
As $s \to 0$, this converges pointwise to $\min(y,b)$.

\paragraph{Soft Clipping Construction.}
A fully compositional smooth approximation to $\min(\max(x,a), b)$ would apply the smooth minimum to the smooth maximum:
\[
\tilde y = s \log\!\left(e^{x/s} + e^{a/s}\right),
\]
\[
\mathrm{clip}_s(x;a,b)
= \tilde y - s \log\!\left(e^{(\tilde y - b)/s} + 1\right).
\]

For computational simplicity and improved numerical stability, we adopt the slightly simplified form:
\begin{align}
\mathrm{clip}_s(x;a,b)
= s \log\!\left(e^{x/s} + e^{a/s}\right)
- s \log\!\left(e^{(x-b)/s} + 1\right).
\end{align}

Note that the second term uses $x$ instead of $\tilde y$. We now show that this modification does not affect pointwise convergence to the hard clipping operator.

\paragraph{Convergence Analysis.}
We analyze the limit $s \to 0$ under three cases.
\begin{itemize}
    \item  \textbf{Case 1: $x < a$.}

As $s \to 0$,
\[
s \log(e^{x/s} + e^{a/s}) \to a,
\]
since $e^{a/s}$ dominates.  
Moreover, since $x-b < 0$,
\[
s \log(e^{(x-b)/s}+1) \to 0.
\]
Therefore,
\[
\mathrm{clip}_s(x;a,b) \to a.
\]

\item \textbf{Case 2: $a \le x \le b$.}

As $s \to 0$,
\[
s \log(e^{x/s} + e^{a/s}) \to x,
\]
since $e^{x/s}$ dominates.  
Additionally, since $x-b \le 0$,
\[
s \log(e^{(x-b)/s}+1) \to 0.
\]
Thus,
\[
\mathrm{clip}_s(x;a,b) \to x.
\]

\item \textbf{Case 3: $x > b$.}

As $s \to 0$,
\[
s \log(e^{x/s} + e^{a/s}) \to x.
\]
Since $x-b > 0$,
\[
s \log(e^{(x-b)/s}+1) \to x-b.
\]
Therefore,
\[
\mathrm{clip}_s(x;a,b) \to x - (x-b) = b.
\]
\end{itemize}

\paragraph{Conclusion.}
In all three cases,
\[
\lim_{s\to 0} \mathrm{clip}_s(x;a,b)
= \min(\max(x,a), b).
\]
Hence, although the second smooth term uses $x$ instead of the smooth maximum $\tilde y$, the pointwise convergence to the hard projection remains unchanged. This simplification yields a computationally efficient and numerically stable implementation while preserving differentiability.

In our implementation, we set $a=0$ and $b=1$ to constrain the temporal decay component within $[0,1]$, ensuring bounded influence magnitudes during training. We set the smooth hyperparameter $s=0.1$.

\section{Variance Reduction of the Stratified Monte Carlo Estimator}
\label{app:stratified-mc}

We provide a variance-reduction justification for the stratified Monte Carlo estimator used in Section~\ref{sect:algo}. Consider an inter-event interval $[t_n,t_{n+1}]$ and define its length as
\[
L_n = t_{n+1}-t_n.
\]
Since no event occurs inside this interval, the event history remains fixed within the interval. Therefore, for any fixed model parameters, the integrated intensity term can be written as an integral of a deterministic function over this interval. For example, for a single event type $k$, we may define
\[
g(t) = \lambda_k(t \mid \Hcal_{t_{n+1}}),
\qquad t \in [t_n,t_{n+1}],
\]
or, for the full integrated intensity over all event types, we may define
\[
g(t) = \sum_{k=1}^K \lambda_k(t \mid \Hcal_{t_{n+1}}).
\]
Thus, it suffices to analyze the estimation of
\[
I = \int_{t_n}^{t_{n+1}} g(t)\,dt.
\]

Partition $[t_n,t_{n+1}]$ into $Q$ equal-length, non-overlapping subintervals
$S_1,\ldots,S_Q$, each of length $L_n/Q$. The stratified Monte Carlo estimator draws one sample
$\hat t_q \sim \mathrm{Unif}(S_q)$ from each subinterval and estimates
\[
\widehat I_{\mathrm{strat}}
=
\frac{L_n}{Q}\sum_{q=1}^Q g(\hat t_q).
\]
For comparison, the standard Monte Carlo estimator with the same number of samples draws
$\tilde t_1,\ldots,\tilde t_Q \overset{\mathrm{i.i.d.}}{\sim}
\mathrm{Unif}([t_n,t_{n+1}])$ and estimates
\[
\widehat I_{\mathrm{mc}}
=
\frac{L_n}{Q}\sum_{q=1}^Q g(\tilde t_q).
\]

\begin{lemma}[Variance reduction of stratified Monte Carlo]
\label{lem:stratified-mc-var}
Assume $g(t)$ has finite second moment on $[t_n,t_{n+1}]$. Then both
$\widehat I_{\mathrm{strat}}$ and $\widehat I_{\mathrm{mc}}$ are unbiased estimators of
\[
I = \int_{t_n}^{t_{n+1}} g(t)\,dt.
\]
Moreover,
\[
\mathrm{Var}\!\left[\widehat I_{\mathrm{strat}}\right]
\le
\mathrm{Var}\!\left[\widehat I_{\mathrm{mc}}\right].
\]
The inequality is strict whenever the stratum-wise means of $g$ are not all identical.
\end{lemma}

\begin{proof}
We first show unbiasedness. For the stratified estimator,
\[
\mathbb{E}\!\left[\widehat I_{\mathrm{strat}}\right]
=
\frac{L_n}{Q}\sum_{q=1}^Q
\mathbb{E}_{\hat t_q \sim \mathrm{Unif}(S_q)}[g(\hat t_q)].
\]
Since each $S_q$ has length $L_n/Q$,
\[
\mathbb{E}_{\hat t_q \sim \mathrm{Unif}(S_q)}[g(\hat t_q)]
=
\frac{Q}{L_n}\int_{S_q} g(t)\,dt.
\]
Therefore,
\[
\mathbb{E}\!\left[\widehat I_{\mathrm{strat}}\right]
=
\frac{L_n}{Q}\sum_{q=1}^Q
\frac{Q}{L_n}\int_{S_q} g(t)\,dt
=
\sum_{q=1}^Q \int_{S_q} g(t)\,dt
=
\int_{t_n}^{t_{n+1}} g(t)\,dt
=
I.
\]
Similarly, for the standard Monte Carlo estimator,
\[
\mathbb{E}\!\left[\widehat I_{\mathrm{mc}}\right]
=
\frac{L_n}{Q}\sum_{q=1}^Q
\mathbb{E}_{\tilde t_q \sim \mathrm{Unif}([t_n,t_{n+1}])}[g(\tilde t_q)]
=
L_n \cdot \frac{1}{L_n}
\int_{t_n}^{t_{n+1}} g(t)\,dt
=
I.
\]

We now compare the variances. Define the stratum-wise mean and variance as
\[
\mu_q
=
\mathbb{E}[g(\hat t_q)\mid \hat t_q \in S_q],
\qquad
\sigma_q^2
=
\mathrm{Var}[g(\hat t_q)\mid \hat t_q \in S_q].
\]
Since the samples $\hat t_1,\ldots,\hat t_Q$ are independent across strata,
\[
\mathrm{Var}\!\left[\widehat I_{\mathrm{strat}}\right]
=
\mathrm{Var}\!\left[
\frac{L_n}{Q}\sum_{q=1}^Q g(\hat t_q)
\right]
=
\frac{L_n^2}{Q^2}\sum_{q=1}^Q \sigma_q^2.
\]

For the standard Monte Carlo estimator, let
$\tilde t \sim \mathrm{Unif}([t_n,t_{n+1}])$. Drawing $\tilde t$ uniformly from the whole interval is equivalent to first drawing a stratum index
$J \sim \mathrm{Unif}(\{1,\ldots,Q\})$ and then drawing
$\tilde t \sim \mathrm{Unif}(S_J)$. By the law of total variance,
\[
\mathrm{Var}[g(\tilde t)]
=
\mathbb{E}\!\left[
\mathrm{Var}(g(\tilde t)\mid J)
\right]
+
\mathrm{Var}\!\left(
\mathbb{E}[g(\tilde t)\mid J]
\right).
\]
Using the definitions of $\mu_q$ and $\sigma_q^2$, we obtain
\[
\mathbb{E}\!\left[
\mathrm{Var}(g(\tilde t)\mid J)
\right]
=
\frac{1}{Q}\sum_{q=1}^Q \sigma_q^2,
\]
and
\[
\mathrm{Var}\!\left(
\mathbb{E}[g(\tilde t)\mid J]
\right)
=
\frac{1}{Q}\sum_{q=1}^Q(\mu_q-\bar\mu)^2,
\qquad
\bar\mu = \frac{1}{Q}\sum_{q=1}^Q \mu_q.
\]
Therefore,
\[
\mathrm{Var}[g(\tilde t)]
=
\frac{1}{Q}\sum_{q=1}^Q \sigma_q^2
+
\frac{1}{Q}\sum_{q=1}^Q(\mu_q-\bar\mu)^2.
\]
Since $\widehat I_{\mathrm{mc}}$ averages $Q$ i.i.d. samples from the full interval,
\[
\mathrm{Var}\!\left[\widehat I_{\mathrm{mc}}\right]
=
\frac{L_n^2}{Q}\mathrm{Var}[g(\tilde t)].
\]
Substituting the expression above gives
\[
\mathrm{Var}\!\left[\widehat I_{\mathrm{mc}}\right]
=
\frac{L_n^2}{Q^2}\sum_{q=1}^Q \sigma_q^2
+
\frac{L_n^2}{Q^2}\sum_{q=1}^Q(\mu_q-\bar\mu)^2.
\]
Comparing this with the variance of the stratified estimator yields
\[
\mathrm{Var}\!\left[\widehat I_{\mathrm{mc}}\right]
=
\mathrm{Var}\!\left[\widehat I_{\mathrm{strat}}\right]
+
\frac{L_n^2}{Q^2}\sum_{q=1}^Q(\mu_q-\bar\mu)^2.
\]
The second term is nonnegative, and hence
\[
\mathrm{Var}\!\left[\widehat I_{\mathrm{strat}}\right]
\le
\mathrm{Var}\!\left[\widehat I_{\mathrm{mc}}\right].
\]
The inequality is strict whenever the stratum-wise means $\mu_1,\ldots,\mu_Q$ are not all equal.
\end{proof}

Lemma~\ref{lem:stratified-mc-var} shows that standard Monte Carlo contains both within-stratum variability and between-stratum variability. In contrast, stratified Monte Carlo removes the between-stratum component by forcing one sample to be drawn from each subinterval. This is particularly useful when the conditional intensity varies systematically over an inter-event interval. In such cases, uniform sampling over the entire interval may concentrate samples in limited regions, whereas stratification enforces coverage across the interval and better captures local variation of the intensity function.

The same argument applies to the full likelihood integral by choosing
$g(t)=\sum_{k=1}^K \lambda_k(t\mid \Hcal_{t_{n+1}})$ on each interval. Moreover, under standard differentiability and finite-variance conditions, the same variance-reduction intuition also applies componentwise to stochastic gradient estimates, since gradients of the integral terms can be written as integrals of the corresponding derivative functions.


 \section{Dataset Details}\label{sect:dataset} 
\begin{itemize}
\item \textbf{MIMIC}~\citep{johnson2016mimic}. This dataset contains de-identified ICU clinical visit records spanning seven years. Each visit is treated as an event, and the diagnosis category defines the event type, resulting in 
$K=75$ types. The average sequence length is 3. We use 527 sequences for training, 58 for validation, and 65 for testing.

\item \textbf{Amazon}~\citep{mcauley2018amazon}. This dataset consists of time-stamped user product review events from January 2008 to October 2018. Each event corresponds to a reviewed product category, yielding \(K=16\) event types. We extract 5,200 active users, resulting in 6,454 training, 922 validation, and 1,851 test sequences. The average sequence length is 45 (minimum 14, maximum 94).

\item \textbf{Retweet}~\citep{zhou2013learning}. This dataset contains time-stamped retweet sequences categorized into \(K=3\) types: “small,” “medium,” and “large” users, defined by follower counts. We use 9,000 sequences for training, 1,535 for validation, and 1,520 for testing. The average sequence length is 41, with a maximum of 97.

\item \textbf{Taxi}~\citep{whong2014foiling}. This dataset records time-stamped taxi pick-up and drop-off events across the five boroughs of New York City. Each (borough, pick-up/drop-off) pair defines an event type, yielding \(K=10\) types. We randomly sample 2,000 drivers and split them into 1,400 training, 200 validation, and 400 test sequences.

\item \textbf{StackOverflow}~\citep{jure2014snap}. This dataset contains two years of user badge-award events from a question-answering platform, with \(K=22\) badge types. The train/validation/test split is 4,066 / 451 / 1,126 sequences. The average sequence length is 59.

\end{itemize}
The MIMIC dataset was obtained from~\url{https://drive.google.com/drive/folders/1DJZjYv1eWcmK55xmRk4jVRSB1Alx3vY4}
, as provided in~\citep{mei2017neural,zuo2020transformer}. All other datasets were downloaded from the EasyTPP repository at~\url{https://drive.google.com/drive/u/0/folders/1f8k82-NL6KFKuNMsUwozmbzDSFycYvz7}.

\section{Hyperparameter Settings}\label{sect:hyperparameter}

Throughout the experiments, we set the event-type embedding dimension to 4, the smooth parameter for the soft-clipping transformation $s=0.1$, and the learning rate to $10^{-3}$. For the interaction network $\psi$, we used the GELU activation function. For the monotonic temporal network $\phi$, we adopted the SoftPlus activation (with $\beta = 1$) and enforced non-negativity of the linear weights by applying a SoftPlus transformation to the unconstrained parameters. 
We set the number of segments to $Q=4$ for the stratified Monte Carlo estimator. We evaluated $Q \in \{1,2,4\}$ and observed comparable final performance across all settings; however, $Q=4$ exhibits slightly faster progress during the early training phase and lower variability. See Section~\ref{sect:ablation} for the corresponding ablation study.


For the synthetic datasets \textbf{PP1} and \textbf{PP2}, as well as the simulated supply-chain dataset, both $\psi$ and $\phi$ were fixed to two hidden layers with 16 neurons per layer. For the real-world datasets, we varied the number of hidden layers in $\{1,2,3\}$ and the number of neurons per layer in $\{4,8,16,32\}$. The mini-batch size is set to 16. Empirically, two hidden layers with 16 neurons each consistently provided robust performance across datasets. On the real-world and supply-chain datasets, we used ELU-plus-one as the positive link function in~\eqref{eq:lam-k} to ensure non-negativity of the conditional intensity. 
Real-world experiments were run on a Linux workstation with an NVIDIA H200 GPU, while synthetic experiments were run on a workstation with an NVIDIA RTX 4090 GPU.

\begin{figure}[!h]
	\centering
	\includegraphics[width=0.9\linewidth]{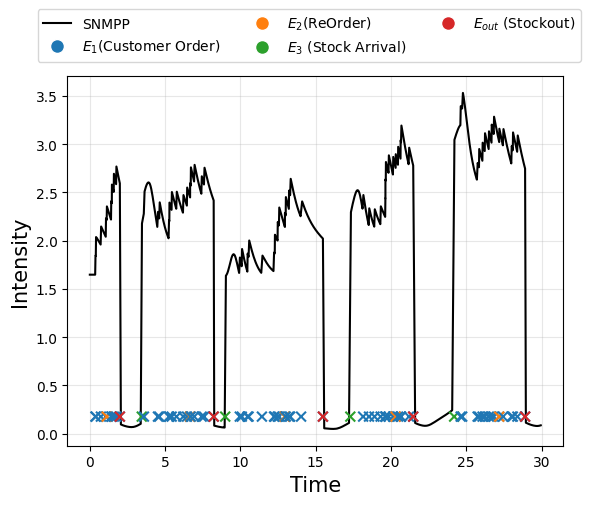}
    \caption{\small Conditional intensity function of $E_1$ (customer order) inferred by \ours\ on a sequence from the simulated supply-chain system.}
	\label{fig:intensity-sc}
\end{figure}

\begin{table}[t]
\centering
\small
\caption{Base and delay parameter estimation on \textbf{PP1} and \textbf{PP2}.}
\label{tab:pp1-pp2}
\begin{subtable}[t]{\linewidth}
\centering
\caption{\textbf{PP1}}
\begin{tabular}{ccccccc}
\toprule
& Base intensity ($E_1$) & Base intensity ($E_2$) 
& $d_{E_1 \rightarrow E_1}$ 
& $d_{E_1 \rightarrow E_2}$ 
& $d_{E_2 \rightarrow E_1}$ 
& $d_{E_2 \rightarrow E_2}$ \\
\midrule
Ground-Truth & 0.5  & 0.05 & 0 & 1.0 & 0 & 0 \\
\ours        & 0.4982 & 0.0544 & 0.0011 & 0.9923 & 0.0045 & 0.0010 \\
\bottomrule
\end{tabular}
\end{subtable}

\vspace{0.5em}

\begin{subtable}[t]{\linewidth}
\centering
\caption{\textbf{PP2}}
\begin{tabular}{ccccccc}
\toprule
& Base intensity ($E_1$) & Base intensity ($E_2$) 
& $d_{E_1 \rightarrow E_1}$ 
& $d_{E_1 \rightarrow E_2}$ 
& $d_{E_2 \rightarrow E_1}$ 
& $d_{E_2 \rightarrow E_2}$ \\
\midrule
Ground-Truth & 0.5  & 1.0 & 0 & 1.0 & 0 & 0 \\
\ours        & 0.5011 & 0.9999 & 0.0034 & 0.9970 & 0.0023 & 0.0002 \\
\bottomrule
\end{tabular}
\end{subtable}
\end{table}
 \section{Supply-Chain System Simulation}\label{app:supplychain}
We simulate event sequences using a hidden-state inventory generator that follows operational constraints and decision rules rather than a temporal point process. The model only observes the resulting event times and types.

\subsection{Event types}
The system produces four event types:
(i) customer order $E_1$,
(ii) replenishment order $E_2$,
(iii) stock arrival $E_3$,
and (iv)  recorded stockout $E_{\text{out}}$, which is emitted when inventory first reaches zero. Their physical meanings are summarized in Table~\ref{tb:supply-chain-events}.
\begin{table}[!ht]
\centering
\small
\caption{Event types in the supply-chain simulation.}\label{tb:supply-chain-events}
\begin{tabular}{lll}
\toprule
\textbf{Symbol} & \textbf{Event Type} & \textbf{Physical Meaning} \\
\midrule
$E_1$ & Customer Order & Inventory consumption \\
$E_2$ & Replenishment Order & Reaction to low inventory \\
$E_3$ & Stock Arrival & Inventory restoration \\
$E_{\text{out}}$ & Stockout Flag & Inventory reaches zero \\
\bottomrule
\end{tabular}
\label{tab:supplychain-events}
\end{table}

\subsection{Hidden state and parameters}
The simulator maintains a hidden inventory level $I(t)$ and a pending restock arrival time $t_{\text{arr}}$ (with at most one pending restock at any time). For each sequence, we initialize $I(0)=I_0$ and specify a reorder threshold $r$, reorder quantity $q$, mean lead time $\mu_L$, and time horizon $T_{\max}$.

The base demand rate is randomized independently for each sequence by sampling
 $\lambda \sim \mathrm{Uniform}(1.5,3.5)$. The remaining parameters are fixed as
\[
T_{\max}=30, \quad I_0=10, \quad r=5, \quad q=15, \quad \mu_L=4.0.
\]

\subsection{Dynamics and physical constraints}
Candidate customer orders are generated by inter-arrival times $\Delta t \sim \mathrm{Exp}(\lambda)$. A candidate order at time $t$ becomes an observed $E_1$ only if $I(t^-)>0$, in which case the inventory decreases by one. When the inventory first reaches zero, the simulator records a stockout event $E_{\text{out}}$ and suppresses subsequent customer orders until the next stock arrival. When $I(t)\le r$ and no restock is pending, a replenishment order $E_2$ is triggered and schedules a delayed arrival $E_3$ at time $t_{\text{arr}} = t + L$, where the lead time is sampled from a truncated normal distribution
$L=\max(0.5,\mathcal{N}(\mu_L,1))$.
At $E_3$, inventory increases by $q$, the pending flag is cleared, and the stockout state is reset.

\subsection{Induced interaction structure}\label{sect:structure}
The generator yields: 
\begin{itemize}
    \item Excitation $E_1\!\rightarrow\!E_2$ via inventory depletion.
    \item Delayed excitation $E_2\!\rightarrow\!E_3$ via lead time.
    \item Inhibition $E_3\!\rightarrow\!E_2$ through replenishment satisfaction (no new reorders while inventory is restored)
    \item Hard gating $E_{\text{out}}\!\rightarrow\!E_1$ by suppressing orders at zero inventory.
\end{itemize}
Algorithm~\ref{alg:supplychain} summarizes the simulation procedure.

\begin{algorithm}[t]
\caption{Supply-chain event simulation. In the experiments, we set $T_{\max}=30$, $I_0=10$, $r=5$, $q=15$, and $\mu_L=4.0$.}
\label{alg:supplychain}
\KwIn{$T_{\max}$, $I_0$, reorder point $r$, reorder quantity $q$, lead time mean $\mu_L$}
Sample demand rate $\lambda \sim \mathrm{Uniform}(1.5,3.5)$\;
$t \leftarrow 0$, $I \leftarrow I_0$, $t_{\mathrm{arr}} \leftarrow \varnothing$, $\texttt{stockout}\leftarrow\texttt{false}$\;
$\Gamma \leftarrow [\,]$\;

\While{$t < T_{\max}$}{
    Sample $\Delta t \sim \mathrm{Exp}(\lambda)$ and set $t_{E_1} \leftarrow t + \Delta t$\;

    \If{$t_{\mathrm{arr}} \neq \varnothing$ \textbf{and} $t_{\mathrm{arr}} < t_{E_1}$}{
        $t \leftarrow t_{\mathrm{arr}}$\;
        Append $(t, E_3)$ to $\Gamma$\;
        $I \leftarrow I + q$, $t_{\mathrm{arr}} \leftarrow \varnothing$, $\texttt{stockout}\leftarrow\texttt{false}$\;
        \textbf{continue}\;
    }

    \If{$t_{E_1} > T_{\max}$}{\textbf{break}}

    $t \leftarrow t_{E_1}$\;

    \If{$I \le 0$}{
        \tcp{Hard gating: suppress customer orders during stockout}
        \textbf{continue}\;
    }

    Append $(t, E_1)$ to $\Gamma$; $I \leftarrow I - 1$\;

    \If{$I = 0$ \textbf{and} $\texttt{stockout}=\texttt{false}$}{
        Append $(t, E_{\mathrm{out}})$ to $\Gamma$; $\texttt{stockout}\leftarrow\texttt{true}$\;
    }

    \If{$I \le r$ \textbf{and} $t_{\mathrm{arr}}=\varnothing$}{
        Append $(t, E_2)$ to $\Gamma$\;
        Sample $L \leftarrow \max(0.5, \mathcal{N}(\mu_L, 1))$\;
        $t_{\mathrm{arr}} \leftarrow t + L$\;
    }
}
\KwOut{Chronologically ordered event sequence $\Gamma$}
\end{algorithm}

\section{Ablation Studies}
\subsection{Stratified Monte Carlo Estimator}\label{sect:ablation}

Since the stratified Monte Carlo estimator is a central algorithmic component of our method, we conducted an ablation study on the number of stratification segments $Q$. Specifically, we used the StackOverflow (SO) dataset, varied $Q \in \{1,2,4\}$, and examined how the next-event time RMSE and event-type accuracy evolve during training. 

In addition, we compared against a simple baseline that, for each sequence, randomly samples a single time point $\hat{t}$ and approximates the likelihood integral over the entire sequence as
\begin{align}
    \int_0^T \lambda(t|\Hcal_t) \d t \approx T \cdot \lambda(\hat{t}|\Hcal_{\hat{t}}). \label{eq:gmce}
\end{align}
We refer to this baseline as the \emph{Global Monte Carlo Estimator} (GMCE). All methods use a learning rate of $10^{-3}$ and a mini-batch size of 16.

\paragraph{Comparison with GMCE.}
We first compared GMCE with our method using $Q=1$. As shown in Figure~\ref{fig:GMCE}, GMCE achieves a small time-prediction RMSE after the first epoch, but the error subsequently increases sharply before gradually decreasing and stabilizing. For event-type prediction, GMCE exhibits rapid initial improvement comparable to our method with $Q=1$; however, after approximately 25 epochs, its performance begins to deteriorate. In contrast, our method demonstrates substantially more stable convergence behavior. This instability of GMCE is likely due to the significantly higher variance of the global estimator in~\eqref{eq:gmce}, which approximates the entire integral using a single sample, compared to our stratified estimator that approximates the integration over each inter-event interval $(t_n, t_{n+1})$ independently.

\paragraph{Effect of the Number of Segments $Q$.}
We next examined the learning curves of our method under different choices of $Q$. As shown in Figure~\ref{fig:ours-Q}, all choices of $Q$ achieve comparable final predictive performance, with nearly identical validation-selected errors. However, larger $Q$ values exhibit smoother learning curves and reduced variability during training. In particular, $Q=4$ shows slightly faster improvement during the first 10 epochs. The learning curve with $Q=1$ displays more frequent and larger fluctuations, while $Q=2$ lies between the two ends. These results suggest that increasing the number of segments in the stratified Monte Carlo estimator reduces variance and improves training stability, while having limited impact on final predictive accuracy.

\begin{figure*}[!h]
\centering
\begin{subfigure}{.5\textwidth}
  \centering
  \includegraphics[width=\textwidth]{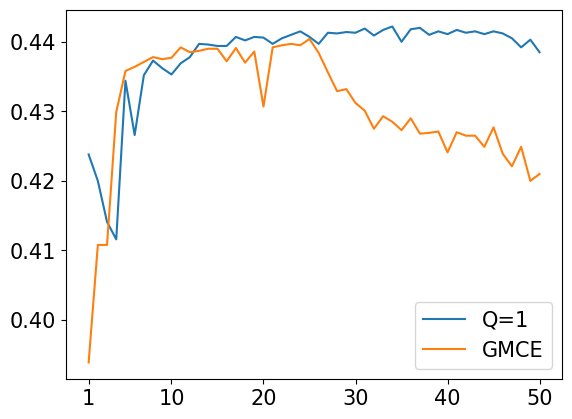}
  \caption{Event type accuracy}
  \label{fig:GMCE-acc}
\end{subfigure}%
\begin{subfigure}{.5\textwidth}
  \centering
  \includegraphics[width=\textwidth]{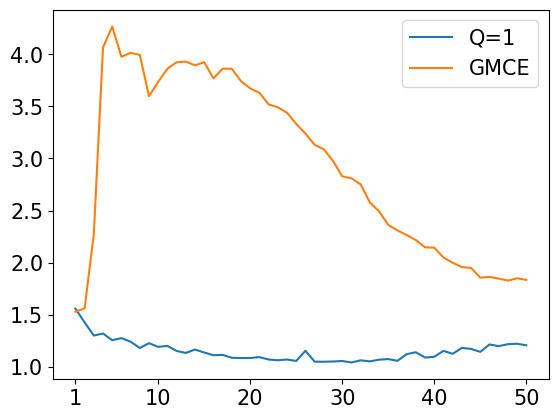}
  \caption{Time RMSE}
  \label{fig:GMCE-rmse}
\end{subfigure}
\caption{\small Next-event time RMSE and event-type prediction accuracy over training epochs for the global Monte Carlo estimator (GMCE) and our stratified Monte Carlo estimator ($Q=1$).}
\label{fig:GMCE}
\end{figure*}

\begin{figure*}[!h]
\centering
\begin{subfigure}{.5\textwidth}
  \centering
  \includegraphics[width=\textwidth]{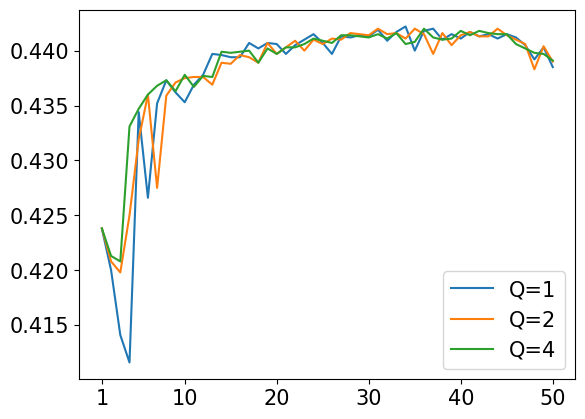}
  \caption{\small Event type accuracy}
  \label{fig:nquad-acc}
\end{subfigure}%
\begin{subfigure}{.5\textwidth}
  \centering
  \includegraphics[width=\textwidth]{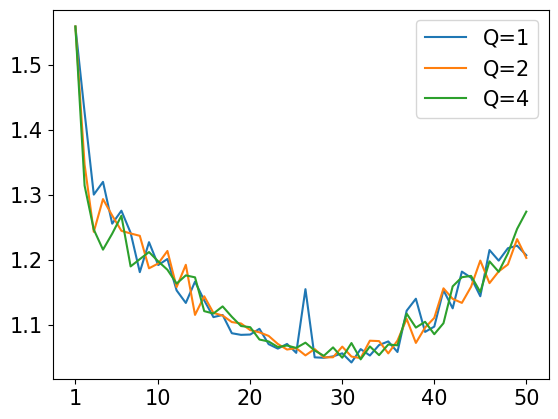}
  \caption{\small Time RMSE}
  \label{fig:nquad-rmse}
\end{subfigure}
\caption{\small Next-event time RMSE and event-type prediction accuracy versus training epochs under different choices of $Q$.}
\label{fig:ours-Q}
\end{figure*}

\subsection{Smoothness Parameter $s$}

Throughout our experiments, we set the smoothness parameter to $s=0.1$ in the soft-clipping transformation~\eqref{eq:soft-clilp}. In this section, we examine the sensitivity of \ours to this choice. We vary $s \in \{0.01, 0.05, 0.1, 0.5, 1.0, 10.0\}$ and evaluate \ours on the StackOverflow (SO) dataset. We track next-event type prediction accuracy and next-event time RMSE over training epochs.

As shown in Figure~\ref{fig:smooth}, all settings generally improve during training in both event-type and event-time prediction. However, the choice of $s$ has a more pronounced effect on event-type prediction accuracy than on event-time RMSE. In Figure~\ref{fig:s-acc}, larger values of $s$ ($s\ge 0.5$) lead to substantial fluctuations in type accuracy after roughly 10 training epochs, with larger $s$ values producing more unstable learning dynamics. In contrast, smaller values $s\in\{0.01,0.05,0.1\}$ yield more stable improvement throughout training. Among these settings, $s=0.1$ is typically near the top of the accuracy curve, suggesting a slightly better stability--performance trade-off than smaller values.

For event-time prediction, the learning curves are less sensitive to $s$, as shown in Figure~\ref{fig:s-rmse}. Nevertheless, very large values such as $s=10.0$ still exhibit noticeable fluctuations after the early training phase. Overall, these results suggest that overly large smoothness values may make the soft-clipping transformation too diffuse, weakening the intended bounded-output behavior and leading to less stable optimization. The choice $s=0.1$ provides a practical balance between smooth differentiability and effective clipping.

\begin{figure*}[!h]
\centering
\begin{subfigure}{.5\textwidth}
  \centering
  \includegraphics[width=\textwidth]{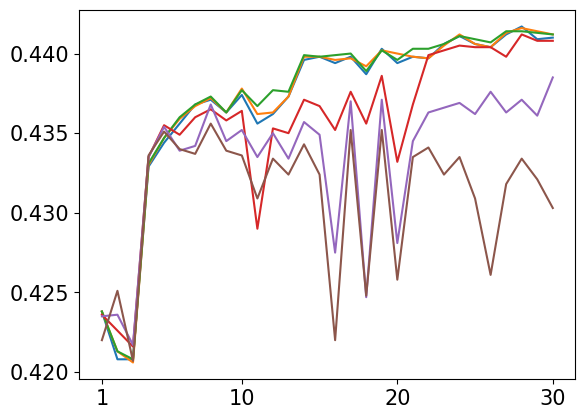}
  \caption{Event type accuracy}
  \label{fig:s-acc}
\end{subfigure}%
\begin{subfigure}{.5\textwidth}
  \centering
  \includegraphics[width=\textwidth]{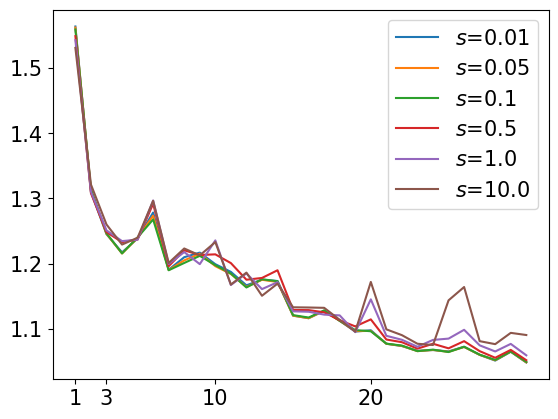}
  \caption{Time RMSE}
  \label{fig:s-rmse}
\end{subfigure}
\caption{\small Next-event time RMSE and event-type prediction accuracy over training epochs with different choices of the smoothness parameter $s$ in the soft-clipping transformation~\eqref{eq:soft-clilp}.}
\label{fig:smooth}
\end{figure*}







\section{Limitations, Discussion, and Future Work}\label{sect:limitation}

\ours is designed to flexibly capture inter-type influence topology, diverse temporal decay patterns of event-level influence strength, and potential delayed peak effects. This structured design provides greater flexibility for interpretable relationship discovery than many classical temporal point process models, such as standard Hawkes processes. Nevertheless, the proposed structural inductive bias does not cover all possible temporal interaction patterns. For example, it may not fully capture multi-modal temporal influence profiles or interactions whose sign changes over time between the same pair of event types.

This limitation reflects a broader trade-off between expressivity and interpretability. Fully flexible black-box neural point process models can represent highly complex temporal dependencies, but their latent representations often obscure the underlying interaction structure. In contrast, overly restrictive structured models may fail to capture important dynamics.  Our simulation studies in Section~\ref{sect:expr-syn}, together with empirical comparisons against black-box neural point process baselines, suggest that an appropriate inductive bias can improve interpretability while maintaining strong predictive performance. In particular, a model that is flexible enough to capture signed interactions with delayed and temporally decaying effects, while still being sufficiently structured to regularize the learning problem, can support more reliable relationship discovery from data.

An important direction for future work is to further expand the expressivity of our structured neural point processes while preserving interpretability. One natural extension is to allow asymmetric temporal influence profiles, where the influence trajectory before and after the estimated delay parameter is modeled by different neural components. Such a design could capture richer influence shapes around the delayed peak. However, it may also introduce additional training challenges, especially when many delay parameters are close to zero; in such cases, the network responsible for modeling the pre-delay influence region may receive limited training signal. A possible solution is to first train a unimodal influence profile, as in the current model, and then fine-tune an additional asymmetric component. We leave this extension for future work.


\newpage
\section*{NeurIPS Paper Checklist}

The checklist is designed to encourage best practices for responsible machine learning research, addressing issues of reproducibility, transparency, research ethics, and societal impact. Do not remove the checklist: {\bf The papers not including the checklist will be desk rejected.} The checklist should follow the references and follow the (optional) supplemental material.  The checklist does NOT count towards the page
limit. 

Please read the checklist guidelines carefully for information on how to answer these questions. For each question in the checklist:
\begin{itemize}
    \item You should answer \answerYes{}, \answerNo{}, or \answerNA{}.
    \item \answerNA{} means either that the question is Not Applicable for that particular paper or the relevant information is Not Available.
    \item Please provide a short (1--2 sentence) justification right after your answer (even for \answerNA). 
\end{itemize}

{\bf The checklist answers are an integral part of your paper submission.} They are visible to the reviewers, area chairs, senior area chairs, and ethics reviewers. You will also be asked to include it (after eventual revisions) with the final version of your paper, and its final version will be published with the paper.

The reviewers of your paper will be asked to use the checklist as one of the factors in their evaluation. While \answerYes{} is generally preferable to \answerNo{}, it is perfectly acceptable to answer \answerNo{} provided a proper justification is given (e.g., error bars are not reported because it would be too computationally expensive'' or ``we were unable to find the license for the dataset we used''). In general, answering \answerNo{} or \answerNA{} is not grounds for rejection. While the questions are phrased in a binary way, we acknowledge that the true answer is often more nuanced, so please just use your best judgment and write a justification to elaborate. All supporting evidence can appear either in the main paper or the supplemental material, provided in appendix. If you answer \answerYes{} to a question, in the justification please point to the section(s) where related material for the question can be found.

IMPORTANT, please:
\begin{itemize}
    \item {\bf Delete this instruction block, but keep the section heading ``NeurIPS Paper Checklist"},
    \item  {\bf Keep the checklist subsection headings, questions/answers and guidelines below.}
    \item {\bf Do not modify the questions and only use the provided macros for your answers}.
\end{itemize}


\begin{enumerate}

\item {\bf Claims}
    \item[] Question: Do the main claims made in the abstract and introduction accurately reflect the paper's contributions and scope?
    \item[] Answer: \answerYes{} 
    \item[] Justification: Our abstract and introduction are align with our paper's contributions and scope.

\item {\bf Limitations}
    \item[] Question: Does the paper discuss the limitations of the work performed by the authors?
    \item[] Answer: \answerYes{} 
    \item[] Justification: We discussed the limitation in Section 6 and Appendix. G.

\item {\bf Theory assumptions and proofs}
    \item[] Question: For each theoretical result, does the paper provide the full set of assumptions and a complete (and correct) proof?
    \item[] Answer: \answerYes{} 
    \item[] Justification: Our theory and proof in Appendix. A and B provide the full set of assumptions and a complete and correct proof.

    \item {\bf Experimental result reproducibility}
    \item[] Question: Does the paper fully disclose all the information needed to reproduce the main experimental results of the paper to the extent that it affects the main claims and/or conclusions of the paper (regardless of whether the code and data are provided or not)?
    \item[] Answer: \answerYes{} 
    \item[] Justification: Our experiment results are reproducible with our code.

\item {\bf Open access to data and code}
    \item[] Question: Does the paper provide open access to the data and code, with sufficient instructions to faithfully reproduce the main experimental results, as described in supplemental material?
    \item[] Answer: \answerNA{} 
    \item[] Justification: We will release our code repository once our work is published.

\item {\bf Experimental setting/details}
    \item[] Question: Does the paper specify all the training and test details (e.g., data splits, hyperparameters, how they were chosen, type of optimizer) necessary to understand the results?
    \item[] Answer: \answerYes{} 
    \item[] Justification:The experiment settings and details are discussed in Appendix. D.

\item {\bf Experiment statistical significance}
    \item[] Question: Does the paper report error bars suitably and correctly defined or other appropriate information about the statistical significance of the experiments?
    \item[] Answer: \answerYes{} 
    \item[] Justification: Our main results included standard deviation.

\item {\bf Experiments compute resources}
    \item[] Question: For each experiment, does the paper provide sufficient information on the computer resources (type of compute workers, memory, time of execution) needed to reproduce the experiments?
    \item[] Answer: \answerYes{} 
    \item[] Justification: We discussed the hardware spec in Appendix. D.
    
\item {\bf Code of ethics}
    \item[] Question: Does the research conducted in the paper conform, in every respect, with the NeurIPS Code of Ethics \url{https://neurips.cc/public/EthicsGuidelines}?
    \item[] Answer: \answerYes{} 
    \item[] Justification: Our work follows NeurIPS Code of Ethics.

\item {\bf Broader impacts}
    \item[] Question: Does the paper discuss both potential positive societal impacts and negative societal impacts of the work performed?
    \item[] Answer: \answerNA{} 
    \item[] Justification: Our work does not have any societal impacts.
    
\item {\bf Safeguards}
    \item[] Question: Does the paper describe safeguards that have been put in place for responsible release of data or models that have a high risk for misuse (e.g., pre-trained language models, image generators, or scraped datasets)?
    \item[] Answer: \answerNA{} 
    \item[] Justification: Our work poses no such risks.

\item {\bf Licenses for existing assets}
    \item[] Question: Are the creators or original owners of assets (e.g., code, data, models), used in the paper, properly credited and are the license and terms of use explicitly mentioned and properly respected?
    \item[] Answer: \answerNA{} 
    \item[] Justification: Our work does not use existing assets.

\item {\bf New assets}
    \item[] Question: Are new assets introduced in the paper well documented and is the documentation provided alongside the assets?
    \item[] Answer: \answerNA{} 
    \item[] Justification: Our work does not release new assets.

\item {\bf Crowdsourcing and research with human subjects}
    \item[] Question: For crowdsourcing experiments and research with human subjects, does the paper include the full text of instructions given to participants and screenshots, if applicable, as well as details about compensation (if any)? 
    \item[] Answer: \answerNA{} 
    \item[] Justification: Our work does not involve crowdsourcing nor research with human subjects.

\item {\bf Institutional review board (IRB) approvals or equivalent for research with human subjects}
    \item[] Question: Does the paper describe potential risks incurred by study participants, whether such risks were disclosed to the subjects, and whether Institutional Review Board (IRB) approvals (or an equivalent approval/review based on the requirements of your country or institution) were obtained?
    \item[] Answer: \answerNA{} 
    \item[] Justification: Our work does not involve crowdsourcing nor research with human subjects.

\item {\bf Declaration of LLM usage}
    \item[] Question: Does the paper describe the usage of LLMs if it is an important, original, or non-standard component of the core methods in this research? Note that if the LLM is used only for writing, editing, or formatting purposes and does \emph{not} impact the core methodology, scientific rigor, or originality of the research, declaration is not required.
    \item[] Answer: \answerNA{} 
    \item[] Justification: Our paper does not involve usage of LLM.

\end{enumerate}

\end{document}